\documentclass[runningheads]{llncs}
\usepackage[T1]{fontenc}
\usepackage{graphicx}
\usepackage{booktabs}
\usepackage[misc]{ifsym}
\newcommand{\corr}{(\Letter)}

\usepackage{microtype}
\usepackage{graphicx}
\usepackage{url}
\usepackage{amsfonts}       
\usepackage{nicefrac}       
\usepackage{wrapfig}
\usepackage{algorithm}
\usepackage[noend]{algorithmic}
\usepackage{amsmath}
\usepackage{amssymb}
\usepackage{mathtools}
\usepackage{dsfont}
\usepackage{bm}
\usepackage[dvipsnames]{xcolor}
\usepackage{def}

\definecolor{darkred}{rgb}{0.7,0,0}
\usepackage[colorlinks=true, linkcolor=blue, citecolor=darkred]{hyperref}

\newtheorem{assumption}{Assumption}
\newtheorem{result}{Result}

\usepackage{tikz}

\begin{document}

\title{Active Preference Optimization for Sample Efficient RLHF}

\titlerunning{Active Preference Optimization for Sample Efficient RLHF}

\author{Nirjhar Das\inst{1} \corr \and Souradip Chakraborty\inst{2} \and Aldo Pacchiano\inst{3} \and Sayak Ray Chowdhury\inst{4}}

\authorrunning{N. Das et al.}

\institute{
Indian Institute of Science, Bangalore, India \email{nirjhardas@iisc.ac.in} \and
University of Maryland, College Park, USA \email{schakra3@umd.edu} \and
Boston University, USA \email{pacchian@bu.edu} \and
Indian Institute of Technology Kanpur, India \email{sayakrc@cse.iitk.ac.in}
}

\maketitle              
\begin{abstract}
  Large Language Models (LLMs) aligned using Reinforcement Learning from Human Feedback (RLHF) have shown remarkable generation abilities in numerous tasks. However, collecting high-quality human preferences creates costly bottlenecks in practical deployments, and hence, training data are often budgeted. In these scenarios, it is crucial to collect training data (e.g., contexts, a pair of generations for each context, and a preference indicating which generation is better) carefully, yet most of the existing methods sample contexts uniformly at random from a given collection. Given this, under the Bradley-Terry-Luce preference model and with a small budget of training data, we show that uniform sampling of contexts could lead to a policy (i.e., an aligned model) that suffers a constant sub-optimality gap from the optimal policy. This highlights the need for an adaptive context sampling strategy for effective alignment under a small sample budget. To address this, we reformulate RLHF within the contextual preference bandit framework, treating generations as actions, and give a nearly complete characterization of the sub-optimality gap in terms of both lower and upper bounds. First, when the action set is a $d$-dimensional hypercube and the number of samples is $T$, we show an $\Omega(d/\sqrt{T})$ lower bound. Next, we propose an algorithm, \emph{Active Preference Optimization} (\texttt{APO}), that iteratively collects preferences for the most uncertain contexts. We show that the sub-optimality gap of the policy learned via \texttt{APO} matches the lower bound up to a log factor and a non-linearity constant. Finally, we perform experiments on practical datasets to validate \texttt{APO}'s efficacy over existing methods, establishing it as a sample-efficient and cost-effective solution for LLM alignment.
\end{abstract}

\section{Introduction}
\label{section:introduction}
Reinforcement Learning from Human Feedback (RLHF) has proven highly effective in aligning Large Language Models (LLMs) with human preferences \cite{christiano2017deep,ouyang2022training,glaese2022improving}. This approach involves collecting extensive data, each comprising a context (e.g., a movie review), a pair of generations (e.g., completions of the review), and a preference indicating which generation is better than the other. First, a reward model is trained to classify preferred generations, and subsequently, a language model policy is trained using RL (e.g., Proximal Policy Optimization~\cite{schulman2017proximal}) to output high-reward generations while minimizing divergence from a reference policy. Given the practical success, recent theoretical advances have been made in training reward models as well as aligning policies from pairwise comparisons~\cite{pacchiano2021dueling,chen2022human,zhu2023principled}. In these settings, the learner doesn't have any control over selecting the contexts, and the aim is to minimize cumulative loss or regret due to not knowing the ground-truth reward or the optimal policy in advance. However, in the case of aligning LLMs using RLHF, the learner has control over both the contexts and actions, i.e., the contexts and generations for which preference data needs to be collected, yet most of the existing RLHF algorithms pick contexts uniformly at random from a given pool~\cite{stiennon2020learning,ouyang2022training}. This is followed by first generating a pair of responses for each sampled prompt based on a supervised fine-tuned (SFT) policy and then sending all the pairs to a human labeler to collect preferences. 

The success of RLHF hinges on the quality of human preferences. This could create costly bottlenecks in practical deployments since high-quality preferences are expensive to collect as this demands a certain level of expertise from labelers. Hence, there is often a budget on the number of contexts and associated generation pairs that could be sent to expert labelers for comparison. While uniform sampling of contexts as a simple approach has been proven effective for aligning LLMs so far, one is bound to ask whether this is a good enough strategy, especially given a fixed budget for labeling. Or do we need potentially more involved sampling strategies to deliver better model alignment under budget constraints? Such algorithms need to be sample efficient as high-quality samples are expensive to obtain, while they should not compromise on the performance of the aligned policy. This work takes a step in developing theory and algorithms for RLHF under a small sample budget.  

\subsection{Overview of Main Results}

We first show that the naive way of collecting preferences by choosing contexts uniformly at random can lead to wastage of samples under the Bradley-Terry-Luce (BTL) preference model characterized by a finite dimensional parameter~\cite{bradley1952rank,luce2012individual}. 
\begin{result}[Constant sub-optimality under uniform context sampling] There exists an instance of the alignment problem for which an algorithm that (i) collects preferences by sampling contexts uniformly at random, (ii) learns a reward model by maximizing likelihood of the preferred generations, and (iii) trains a greedy policy w.r.t. the learnt reward model suffers an $\Omega(1)$ sub-optimality gap\footnote{Measured by the maximum difference between latent rewards of the optimal policy and the trained policy over the context set.} with high probability when the sample budget is smaller than number of contexts.
\end{result}
To the best of our knowledge, this is the first provable negative result for the alignment performance of RLHF algorithms that sample contexts uniformly under a small sample budget. This necessitates designing of RLHF algorithms that adaptively sample contexts, with an aim to improve alignment of the learned policy with human preferences. Our next result benchmarks the alignment performance of any RLHF algorithm under the BTL preference model when there is no restriction on how contexts can be sampled. 

\begin{result}[Lower bound on sub-optimality for any sampling strategy]
For any RLHF algorithm there exists an instance of the alignment problem for which the policy that the algorithm outputs after collecting $T$ samples (contexts, generations and preferences) would suffer a sub-optimality gap $\Omega\big(d/\sqrt{T}\big)$, where $d$ is the dimension of the parameter characterizing the BTL model.
\end{result}
This is the first theoretical lower bound on the alignment performance of RLHF algorithms, which has been crucially missing in the literature. This result effectively eliminates the possibility of achieving better than sub-linear convergence under a finite sample budget. Next, we propose an algorithm -- \emph{Active Preference Optimization} (\texttt{APO}) --
that achieves this sub-optimality gap by iteratively sampling the most uncertain contexts and collecting preferences for their generation pairs.
\begin{result}[Upper bound on sub-optimality for \texttt{APO}]
The sub-optimality gap of the policy learned via \texttt{APO} after collecting $T$ samples scales as $\tilde{O}\big(d\sqrt{\kappa/T}\big)$ with high probability, where $d$ is the parameter dimension and $\kappa$ is a problem-dependent nonlinearity constant.
\end{result}

Next, we generalize our result from parameterized BTL model to non-parametric preference models with function approximation. We propose an analogue of \texttt{APO} albeit with non-trivial modifications, namely $\texttt{APO-Gen}$, that achieves a similar sub-optimality gap (see Subsection~\ref{subsection:general-function-approximation} and Appendix~\ref{section:general-function-approximation}).
This is the first known upper bound on the alignment performance of an active context selection strategy under generic preference models, which recovers the result for the parameterized BTL model as a special case.
\begin{result}[Upper bound on sub-optimality under general preferences]
The sub-optimality gap of the policy learned via \texttt{APO-Gen} after collecting $T$ samples scales as $\tilde{O}\big(\sqrt{d_\cE\log(\cN T)/T}\big)$ with high probability, where $\cN$ and $d_\cE$ measure the complexity of the underlying preference model.
\end{result}

\textbf{Empirical evidence.} For practical purposes, we propose a batch version of \texttt{APO} to make it computationally more efficient (see Section~\ref{sec:exps}). We experiment with GPT-2 on IMDb sentiment dataset~\cite{imdb_dataset} and demonstrate significant improvement in LLM alignment over uniform context sampling and prior baselines~\cite{mehta2023sample,muldrew2024active}. We show similar improvement in the performance of the aligned policy on Anthropic-HH dataset~\cite{anthropic_dataset} with Gemma-2b. Our work contributes towards a sample-efficient and practical solution to preference data collection for RLHF.

\subsection{Comparison with Prior Work}
Active learning in the context of Preference-based Reinforcement Learning (PbRL)~\cite{pacchiano2021dueling,chen2022human}, which is used as a theoretical framework for RLHF, has received some attention recently. In PbRL literature, the problem of learning the reward function by actively querying the human labeller has been considered in~\cite{Sadigh2017ActivePL,carvalho2024deep,pmlr-v139-lee21i,ji2024reinforcement,Zhan2023HowTQ}. The work~\cite{pmlr-v139-lee21i} provides variance and entropy-based heuristics for learning the optimal policy without providing any provable guarantee. On the other hand, in~\cite{Sadigh2017ActivePL}, the authors design an algorithm for learning the reward function by actively synthesizing trajectories via a volume removal scheme over the distribution of the unknown parameter. They show that under certain strong assumptions, their algorithm makes progress in reducing the uncertainty over the parameter distribution. Similarly,~\cite{carvalho2024deep} also aims at actively learning the reward function using model epistemic uncertainty as well as the entropy estimate of acquiring a data point but does not provide any provable guarantee. Hence, both of these are orthogonal to our work since they consider learning the reward function, whereas we focus on learning a policy - while learning a good reward function is sufficient for learning a good policy, it is not at all necessary.

In~\cite{Zhan2023HowTQ}, the authors propose a pure exploration strategy for the PbRL problem that finds an $\varepsilon$-optimal policy with $O(\kappa^2 d^2/\varepsilon^2)$ queries to the human labeler.

Instead of modeling RLHF as a PbRL problem, we model it as a contextual dueling bandit problem, where contexts model prompts and actions model generations of an LLM. This necessitates a strategy for not only actively selecting actions but also selecting contexts actively. This is a major point of departure from most active-learning based PbRL works and from pure exploration in dueling bandits literature~\cite{even-dar2006action,pmlr-v238-maiti24a}. It is unclear apriori what should be the optimal strategy to select the context, towards which we show that a design-matrix-based exploration bonus is sufficient (see Section~\ref{section:logistic-bandit}). Moreover, a direct comparison shows that our result (Theorem~\ref{theorem:logistic-case-regret-bound}) is tighter in terms of $\kappa$. In~\cite{ji2024reinforcement}, the authors consider an active learning-based approach to regret minimization in the contextual dueling bandits. Again,~\cite{ji2024reinforcement} does not address the question of context selection but rather allows contexts to be adversarially presented to the learner, who only chooses action given the context and whether to observe the labeler feedback.

Existing works closest to ours are~\cite{muldrew2024active,mehta2023sample}, which also investigate the problem of actively selecting prompts and generations for RLHF in LLMs. \cite{muldrew2024active} proposes an algorithm that actively selects contexts using a heuristic based on generation uncertainty, but they do not give any theoretical guarantee for the proposed method. \cite{mehta2023sample} proves a sub-optimality gap bound for an active context selection strategy that goes down sub-linearly with the number of samples. However, they assume a strong restrictive condition on the preference model, which \emph{doesn't hold in general} for the BTL model. We remove this restrictive assumption and provide an improved guarantee (see Remark~\ref{remark:comparison-with-mehta} for a detailed comparison).

\section{Problem Setup}
\label{sec:preliminaries}
We have a set of contexts $\cX$ and a set of possible actions per context $\cA$. To learn using preference feedback, the agent selects $x \in \cX$ and $a,a' \in \cA$ to present to a human labeller, who then reveals a binary preference $y$ that takes value $1$ if $a$ wins over $a'$ and $0$ otherwise. We assume that given $(x,a,a')$, $y$ is sampled from the Bradley-Terry-Luce (BTL) preference model~\cite{bradley1952rank,luce2012individual} with $r^*$ as the latent (unknown) reward function, i.e.,
\begin{align*}
\mathbb{P}\left[y\!=\!1 |x, a, a';r^*\right] = \frac{\exp(r^*(x,a))}{\exp(r^*(x,a)) + \exp(r^*(x,a'))}~,
\end{align*}
The goal of the agent is to first learn $r^*$ over 
$T$ rounds of sequential interaction with the labeller, collecting dataset $\cD=(x_s,a_s,a'_s,y_s)_{s=1}^{T}$, and then employ the learned reward to train a policy $\pi:\cX \to \cA$, which will eventually fetch high latent rewards $r^*(x,\pi(x))$.

In this work, we consider linear latent rewards $r^*(x,a) = \phi(x,a)^\top\theta^*$, where $\theta^* \in \RR^d$ is the unknown reward parameter, and $\phi: \cX \times \cA \to \RR^d$ is some known and fixed feature map. For instance, such a $\phi$ can be constructed by
removing the last layer of a pre-trained language model, and in that case, $\theta^*$ corresponds to the weights of the last layer. With this model, for any $\theta \in \RR^d$, one can equivalently write the probability of sampling $y_s = 1$ given $(x_s, a_s, a'_s)$ as 
\begin{align*}
\mathbb{P}\left[ y_s\!=\!1 |x_s, a_s, a'_s;\theta\right] = \sigma\big((\phi(x_s,a_s) \!-\! \phi(x_s,a'_s))^\top \theta \big) = \sigma(z_s^\top \theta)~,
\end{align*}
where $\sigma(w) \!=\! \frac{1}{1+ e^{-w}}$ is the sigmoid function and $z_s = \phi(x_s,a_s)- \phi(x_s,a'_s)$ is the feature differenc of actions $a_s$ and $a'_s$ for context $x_s$.

With this, the latent reward parameter $\theta^*$ is typically estimated by minimizing the binary cross entropy loss (log-loss) \cite{ouyang2022training}, which is equivalent to \emph{maximum likelihood estimation} (MLE). Specifically, At round $t$, the MLE of $\theta^*$ is computed as $\widehat{\theta}_t = \argmin\nolimits_{\theta \in \Theta} \cL_t(\theta)$ using the preference dataset $\{(x_s,a_s,a'_s, y_s)\}_{s=1}^{t-1}$, where log-loss $\cL_t(\theta)$ is given by
\begin{align}\label{eq:log-loss}
\cL_t(\theta) =-\sum\nolimits_{s=1}^{t-1} y_s \log(\sigma(z_s^\top \theta)) + (1-y_s)\log(1-\sigma(z_s^\top\theta)).
\end{align}
The above optimization problem is convex if we let the constraint set $\Theta \subset \mathbb{R}^d$ to be convex, and hence can be solved using standard algorithms~\cite{hazan2016introduction}.

\textbf{Performance Measure.} Our goal is to learn a policy over the collected data $\cD$, which has high rewards or, equivalently, low sub-optimality. Formally, the sub-optimality gap of a policy $\pi_T$ trained on the dataset $\cD$ is defined as
\begin{equation}
\label{eq:simple-regret-definition}
   \!\! R(T) \!=\! \max\nolimits_{x \in \cX} \max\nolimits_{a \in \cA} \left\lbrace r^*(x,a) \!-\! r^*(x,\pi_T(x))\right\rbrace.
\end{equation}
Here, our policy $\pi_T$ competes with the \emph{Condorcet} winner for a given context -- an action that fetches a higher reward than all other actions. The sub-optimality gap is the worst possible difference in rewards over the set of contexts. Prior work \cite{mehta2023sample} competes against the \emph{Borda} winner -- an action that fetches a higher reward on average than 
another randomly chosen action -- a weaker competitor (the \emph{Condorcet} winner is also the \emph{Borda} winner but not the other way around).  
\begin{remark}
To the best of our knowledge, common practical implementations of the RLHF pipeline use the following method: (i) remove the top layer of the LLM and convert it into an encoder, (ii) append a new linear layer on top of it, and (iii) output the logit score. Hence, the linear reward assumption is not restrictive in the sense that we only train the linear layer, keeping the encoder fixed. In practice, however, one needs to train the encoder, and hence, a more general function class needs to be considered. To this end, we generalize this setup to preference models with bounded class complexities, removing the need for explicit linear reward models (see Section~\ref{section:logistic-bandit} and Appendix~\ref{section:general-function-approximation} for details).   
\end{remark}
\section{Lower Bounds}
\label{sec:lower-bound-instance}

We first illustrate the pitfall of a learner who samples contexts uniformly. We characterize such a learner in preference-based learning/RLHF and then show that such a learner can suffer a constant sub-optimality gap under budget constraints.

\begin{definition}[Uniform Learner]
\label{def:uniform-learner}
 Say an algorithm \texttt{Alg} samples $T$ contexts uniformly at random from a set $\cX$ and for each context $x_t$, picks two actions $a_t,a'_t$ from a set $\cA$. For each chosen triplet $(x_t,a_t,a'_t)$, \texttt{Alg} queries the preference model parameterized by $\theta^*$ and observes a stochastic preference $y_t \!\in\! \lbrace 0,1\rbrace$ between the actions. \texttt{Alg} then solves an MLE on this data to obtain $\widehat{\theta}$, and learns a greedy policy with respect to $\widehat{\theta}$. We call such an algorithm \texttt{Alg} a Uniform Learner.
\end{definition}

\begin{theorem}[Lower bound for uniform context sampling]
\label{theorem:lower-bound}
     There exists a problem instance $(\cX, \cA, \theta^*)$ for which the policy learnt by a Uniform Learner \texttt{Alg} under the budget $T \ll \lvert \cX \rvert$ suffers $\Omega(1)$ sub-optimality gap with high probability.
\end{theorem}

\begin{proof}[Sketch]
    We show the result for $d=2$ for simplicity. The main idea is to divide the set of contexts into two groups---\textit{good} and \textit{bad}. Further, every context has only two actions, $a$ and $a'$. The \textit{good} group has a large number of contexts, so the uniform learner mostly samples contexts from this group. For all context $x$ in the \textit{good} group, the feature difference $\phi(x,a)-\phi(x,a') = z_g$ is the same. For \textit{bad} contexts, this feature difference is $z_b$, such that $\langle z_b, z_g \rangle < 0$. Finally, $\theta^*$ is taken as the angle-bisector of $z_g$ and $z_b$. 
     
    From this construction (Fig.~\ref{fig:lower-bound-instance} in Appendix~\ref{appendix:uniform-lower-bound}), we have that $a$ gets a higher reward for every context. Then, we show that the uniform learner only samples context-actions corresponding to $z_g$ when the sample budget $T$ is much smaller than the number of contexts. Under this scenario, the MLE estimate of the uniform learner correctly classifies the reward of $a$ to be higher than that of $a'$ for all the \textit{good} contexts, but it wrongly classifies the rewards for the \textit{bad} contexts. Thus, for \textit{bad} contexts, the uniform learner suffers a constant suboptimality gap. Details are in Appendix~\ref{appendix:uniform-lower-bound}.\hfill$\qed$
\end{proof}

Theorem~\ref{theorem:lower-bound} effectively shows that the uniform learner cannot make efficient use of the sampling budget because its performance may not increase with an increasing budget. Now, it is essential to characterize the limits of learning in this setting. To this end, we prove a lower bound on the sub-optimality gap of any algorithm with no restriction on how contexts and action pairs can be sampled. Note that standard regret lower bounds for dueling~\cite{saha2021optimal} and logistic bandits~\cite{abeille2021instance} are not applicable here because we bound the sub-optimality gap and not regret.
\begin{theorem}[Lower Bound for any sampling strategy]
\label{theorem:active-learning-lower-bound}
Let $\cX$ be a finite set of contexts, $\Theta = \{-\frac{1}{\sqrt{T}},\frac{1}{\sqrt{T}}\}^d $, $\cA = \{-\frac{1}{2}, \frac{1}{2}\}^d$. Then, for any algorithm, there exists a parameter $\theta^* \!\in\! \Theta$ such that sub-optimality gap of a policy learnt by the algorithm after collecting $T$ samples satisfies\footnote{Expectation is over the randomness of $(x_1,a_1,a'_1,y_1,\ldots,y_T)$ under hypothesis $\theta^*$.}
\begin{align*}
\mathbb{E}_{\theta^*}\left[R(T) \right] \ge \Omega\left({d}/{\sqrt{T}}\right)~.
\end{align*}
\end{theorem}
\begin{proof}[Sketch]
    Without loss of generality, choose any $x \in \cX$. Let $\pi_T(x)$ denote the action chosen by the policy learnt using $T$ samples. Define the event $
        \cE_{\theta,i} =
        \{\text{sign}(\pi_{T,i}(x)) \neq \text{sign}(\theta_i)\}
    $, for all $i\in [d]$,
    where $\pi_{T,i}(x)$ and $\theta_i$ are the $i$-th coordinates of $\pi_T(x)$ and the parameter $\theta \in \Theta$, respectively. Note that under the event $\cE_{\theta,i}$, the algorithm suffers a sub-optimality of $\frac{1}{\sqrt{T}}$ for the $i$-th coordinate. We need to lower bound the probability of this event. To this end, let $\theta' \in \Theta$ be such that $\theta'_i = - \theta_i$ and $\theta'_j = \theta_j$ for $j \neq i$. Note that $\cE_{\theta,i}^c = \cE_{\theta',i}$. Therefore, from Lemma~\ref{lemma:bretagnolle-huber-inequality}, we have,
    $
    \PP_\theta[\cE_{\theta,i}] + \PP_{\theta'}[\cE_{\theta',i}] \geq \frac{1}{2} \exp(-D_{KL}(\PP_\theta, \PP_{\theta'}))
    $.

    Next, we need an upper bound on $D_{KL}(\PP_\theta, \PP_{\theta'})$. Using~\cite[Lemma 15.1]{lattimore2020bandit} and Taylor expansion of log-sigmoid, we obtain $D_{KL}(\PP_\theta, \PP_{\theta'}) \leq \frac{1}{8} \EE_\theta \left[ \sum\nolimits_{t=1}^T \langle z_t, \theta - \theta' \rangle^2 \right]$,
    where $z_t = a_t - a'_t$. Since $z_t \in [-1,1]^d$, and $\theta$, $\theta'$ are equal in every coordinate except the $i$-th one, we have,
    $D_{KL} (\PP_\theta, \PP_{\theta'}) \leq \frac{1}{8} \sum_{t=1}^T \EE_\theta[(2 z_{t,i} \theta_i)^2] \leq \frac{1}{2}
    $. Hence, $$\frac{1}{\lvert \Theta \rvert} \sum\nolimits_{\theta \in \Theta} \sum\nolimits_{i=1}^d \PP_\theta[\cE_{\theta,i}] \geq \frac{d}{4} \exp(-\frac{1}{2})~.$$ Therefore, there exists a $\theta^* \in \Theta$ such that $\sum_{i=1}^d \PP_{\theta^*}[\cE_{\theta^*,i}] \geq \frac{d}{4} \exp(-\frac{1}{2})$. Finally, it is easy to see that the expected sub-optimality gap is lower bounded by the expected gap for context $x$. Hence, we have the following chain of inequalities:
    \begin{align*}
        \EE_{\theta^*}[R(T)] &\geq \EE_{\theta^*}\left[\sum\nolimits_{i=1}^d \mathds{1}[\cE_{\theta^*,i}]\cdot 2\lvert \theta^*_i \rvert \right] = \frac{2}{\sqrt{T}} \sum\nolimits_{i=1}^d \PP_{\theta^*}[\cE_{\theta^*,i}] \geq \frac{d \exp(-1/2)}{2\sqrt{T}}
    \end{align*}
    which completes the proof. Details are in Appendix~\ref{appendix:active-lower-bound}.\hfill$\qed$
\end{proof}
To the best of our knowledge, Theorem~\ref{theorem:active-learning-lower-bound} gives the first lower bound for general active-learning algorithms, which was missing in prior work~\cite{mehta2023sample,muldrew2024active}. Now, the immediate question is whether one can design an algorithm that learns a policy whose sub-optimality gap matches this lower bound. In the next section, we present an algorithm \emph{Active Preference Optimization} \texttt{(APO)} that achieves this up to log factors and an instance-dependent non-linear factor.
\section{Our Approach: Active Preference Optimization}
\label{section:logistic-bandit}

\begin{algorithm}[tb]
\caption{\texttt{APO}: Active Preference Optimization}
\label{algo:act-con-sel-logB}
\begin{algorithmic}[1]
    \REQUIRE Context set $\cX$, action set $\cA$, feature map $\phi:\cX \times \cA \rightarrow \RR^d$, regularization $\lambda > 0$, and failure probability $\delta \in (0,1]$. Initialize $\widehat{\theta}_1 = 0$.
    \FOR{$t=1,\dots, T$}
        \STATE Choose the triplet $(x_t, a_t, a'_t)$ using~\eqref{eq:choose-a-a'} and~\eqref{eq:choose-x}. \alglinelabel{line:choice-of-triplet-apo-logistic}
        \STATE Observe preference feedback $y_{t} \!\sim\! \texttt{Ber}\big(\sigma(z_t^\top \theta^*) \big)$, where $z_t=\phi(x_t,a_t)-\phi(x_t,a'_t)$.
        \STATE Compute reward estimate $\widehat{\theta}_{t+1}$ that minimizes the constrained log-loss~\eqref{eq:log-loss}.
        \STATE Compute (scaled) design matrix $H_{t+1}(\widehat{\theta}_{t+1})$ via~\eqref{eq:H-t(theta)-definition}.
    \ENDFOR
    \STATE Compute final policy $\pi_T(x)$ using~\eqref{eq:definition-of-final-policy}.
\end{algorithmic}
\end{algorithm} 
At each round $t$, \texttt{APO} (Algorithm~\ref{algo:act-con-sel-logB}) proceeds by computing the MLE estimate $\widehat{\theta}_t$ based on the data obtained in the past $t-1$ steps~\eqref{eq:log-loss}. Based on $\widehat{\theta}_t$, our goal is to maximize exploration. To do  this, for a context $x  \in \cX$, we compute the uncertainty $b_t(x,a,a')$ for each action $(a,a')$ available for that context and choose the one which maximizes this, i.e., we choose the pair
\begin{equation}\label{eq:choose-a-a'}
(a_t(x),a'_t(x)) = \mathrm{argmax}_{(a,a') \in \cA \times \cA} b_t(x,a,a'),
\end{equation}
where $b_t(x,a,a') = \lVert \phi(x,a) - \phi(x,a') \rVert_{H_t^{-1}(\widehat{\theta}_t)}$ and
$H_t(\widehat\theta_t)$ is a matrix that describes a confidence ellipsoid around the unknown reward parameter $\theta^*$ after $t-1$ steps of data collection. For any $\theta \in \Theta$, this is defined as
\begin{equation}\label{eq:H-t(theta)-definition}
    H_t(\theta) \!=\!  \nabla^2 \cL_t(\theta) \!+\! \lambda \Ib_d\!=\!\sum\nolimits_{s=1}^{t-1} \dot{\sigma}(z_s^\top \theta) z_s z_s^\top \!+\! \lambda \Ib_d~,
    \end{equation}
where $z_s = \phi(x_s,a_s) - \phi(x_s,a'_s)$ is the feature difference for the triplet $(x_s,a_s,a'_s)$.
Intuitively, the confidence ellipsoid keeps shrinking along whichever direction (in $\RR^d$) we decide to explore. Thus, for a given context $x$, choosing the pair $(a_t(x),a'_t(x))$ maximally reduces the uncertainty among all other possible action duels.
However, our algorithm picks not only the action pair that maximizes uncertainty but also the context that increases it the most, i.e.,
\begin{align}
    \label{eq:choose-x}
    x_t = \argmax\nolimits_{x \in \cX} b_t(x, a_t(x), a'_t(x))~.
\end{align}
This is a crucial step in our approach that ensures that the uncertainty of the reward function over all contexts decreases at a fast rate, which in turn ensures a low sub-optimality gap. After $T$ rounds, we define $\theta_T = \frac{1}{T}\sum_{s=1}^{T} \widehat{\theta}_t$ as the average of all the past parameter estimates. Our final policy $\pi_T$ for any context $x \in \cX$ is to play the action that maximizes the reward parameterized by $\theta_T$, i.e.,
\begin{align}
\label{eq:definition-of-final-policy}
    \pi_T(x) =\argmax\nolimits_{a \in \cA(x)}\,\widehat r_T(x,a)=\argmax\nolimits_{a \in \cA(x)}\, \phi(x,a)^\top \theta_T~.
\end{align}
\subsection{Suboptimality Gap of \texttt{APO}}\label{sec:theory}

We make the following assumption, which is standard in the literature~\cite{zhu2023principled,chowdhury2024provably}.
\begin{assumption}[Boundedness]
\label{ass:bound}
(a) $\theta^*$ lies in the set $\Theta = \{\theta \in \RR^d | \inner{\mathbf{1}}{\theta} = 0, \norm{\theta} \le S \}$. (b) Features are bounded, i.e., $\norm{\phi(x,a)} \le 1$, $ \forall\ (x,a) \in \cX\times \cA$.
\end{assumption}
The condition $\inner{\mathbf{1}}{\theta} = 0$ ensures identifiability of $\theta^*$. Now, we define a key quantity that captures learning complexity under the BTL preference model:
\begin{align}\label{eq:kappa}
 \kappa = \max_{x\in \mathcal{X}} \max_{a,a' \in \mathcal{A}} \max_{\theta \in \Theta}\,\frac{1}{\dot{\sigma}(\phi(x,a)^\top \theta \!-\! \phi(x,a')^\top \theta)}~.   
\end{align}
This parameter $\kappa$ specifies difficulty in learning via the worst-case non-linearity in preference feedback. Note that we don't need the knowledge of $\kappa$ beforehand. Next, we present the guarantee that our algorithm enjoys.

\begin{theorem}[Sub-optimality gap of \texttt{APO}]
\label{theorem:logistic-case-regret-bound}
    Let $\delta \in (0,1]$. Under Assumption~\ref{ass:bound}, setting $\lambda \!=\! \frac{1}{4S^2(2+2S)^2}$ and $\gamma = O\left(S\sqrt{d\log\left({S T}/{d}\right) + \log\left({T}/{\delta}\right)}\right)$, the policy $\pi_T$ returned by \texttt{APO}, with probability at least $1 - \delta$, enjoys the suboptimality gap 
    \begin{align*}
        R(T) \leq O \left( \gamma \sqrt{S \log\Big(1+({T}/{\lambda \kappa d})\Big){\kappa d}/{T}}\right)~.
    \end{align*}
\end{theorem}
\textbf{Comparison with lower bound and dependence on $\kappa$.} Theorem~\ref{theorem:logistic-case-regret-bound} implies an $\widetilde O \big(d\sqrt{\kappa/T}\big)$ upper bound on the sub-optimality gap of \texttt{APO} policy. This matches the lower bound of Theorem~\ref{theorem:active-learning-lower-bound} in parameter dimension $d$ and number of samples $T$, implying optimal scaling w.r.t. these two terms (up to a log factor). There remains a gap in characterizing the optimal dependence on the non-linearity parameter $\kappa$, which, in the worst-case, can be exponential in the parameter norm $S$. In the logistic bandit literature, the state-of-the-art regret guarantee is (almost) $\kappa$-independent - the dependence is only in terms independent of $T$~\cite{sawarni2024generalized,lee2024improved}.  We believe the $\sqrt{\kappa}$ dependence is unavoidable in the RLHF setting as the sub-optimality gap is w.r.t. real-valued rewards $r^*(x,a)=\phi(x,a)^\top\theta^*$ instead of the sigmoid rewards $\sigma(\phi(x,a)^\top\theta^*)$ in logistic bandits. Given this, we conjecture that it could be possible to improve the lower bound of Theorem~\ref{theorem:active-learning-lower-bound} to $\Omega\big(d\sqrt{\kappa/T}\big)$. We keep this as an interesting future direction.

\begin{proof}[Sketch]
First, we quantify the error in estimating $\theta^*$ in the following lemma. This is obtained by using a novel inequality derived from the self-concordance property of the sigmoid function (i.e., $\lvert \Ddot{\sigma} \rvert \leq \dot{\sigma}$) and adapting the arguments from~\cite{lee2024improved}. Proof of this lemma is deferred to appendix~\ref{appendix:logistic}.
\begin{lemma}[Estimation error at round $t$] 
\label{lemma:confidence-set-logistic}
Let $\delta \in (0,1]$. Under the hypothesis of Theorem~\ref{theorem:logistic-case-regret-bound}, with probability $\geq 1 - \delta$, for some universal constant $C > 0$,
\begin{align*}
  \lVert \theta^* \!-\! \widehat{\theta}_t \rVert_{H_t(\widehat{\theta}_t)} \!\leq\! CS^{\nicefrac{3}{2}} \sqrt{d\log\left(\nicefrac{S t}{d}\right) + \log\left(\nicefrac{t}{\delta}\right)}~,  
\end{align*}
where $\widehat{\theta}_{t}$ is the estimated reward parameter that minimizes the constrained log-loss~\eqref{eq:log-loss} and $H_{t}(\widehat{\theta}_{t})$ is the (scaled) design matrix~\eqref{eq:H-t(theta)-definition} at round $t$.
\end{lemma}

The proof of Theorem~\ref{theorem:logistic-case-regret-bound} proceeds by upper bounding the sub-optimality gap for every context with the error in parameter estimation times an arm-dependent quantity. Specifically, for context $x$, let $z_{T}(x) = \phi(x,a^*(x)) - \phi(x,\pi_T(x))$ denotes the fetaure difference for the triplet $(x,a^*(x),\pi_T(x))$, where $a^*(x)$ is the optimal action at context $x$. Then, from~\eqref{eq:simple-regret-definition}, the sub-optimality gap for context $x$ can be bounded as $z_T(x)^\top \theta^* \!\leq\! z_{T}(x)^\top \theta^* \!-\! z_{T}(x)^\top \theta_T \!= \!\frac{1}{T}\! \sum\nolimits_{t=1}^T \!z_{T}(x)^\top \!(\theta^*\!\! - \!\widehat{\theta}_t)$. Here, the first inequality is because $\phi(x,\pi_T(x))^\top \theta_T \geq \phi(x,a^*(x))^\top \theta_T$, which follows from definition of $\pi_T$, and so $z_{T}(x)^\top \theta_T \leq 0$. Then, by Cauchy-Schwarz inequality, we get $z_T(x)^\top \theta^*  \leq \frac{1}{T} \!\sum\nolimits_{t=1}^T \!\lVert z_{T}(x) \rVert_{H_t(\widehat{\theta}_t)^{-1}} \lVert \theta^* \!-\! \widehat{\theta}_t \rVert_{H_t(\widehat{\theta}_t)}$.

Now, we apply Lemma~\ref{lemma:confidence-set-logistic} to upper bound $\lVert \theta^* - \widehat{\theta}_t \rVert_{H_t(\widehat{\theta}_t)}$. Next, we note that $\lVert z_{T}(x) \rVert_{H_t(\widehat{\theta}_t)^{-1}} \leq \lVert z_t \rVert_{H_t(\widehat{\theta}_t)^{-1}}$ by the design of our algorithm. To bound this, consider the regularized sample covariance matrix of feature differences, defined as
$ V_t = \sum\nolimits_{s=1}^{t-1} z_s z_s^\top + \kappa \lambda \Ib_d$. Compare this with $H_t(\theta)$, which scales each rank-one component inside the sum by its variance given that the parameter is $\theta$ (see Eq.~\ref{eq:H-t(theta)-definition}). A key relation between these two matrices is that $H_t(\theta) \succcurlyeq V_t/\kappa$. Using this, we 
upper bound $\lVert z_t \rVert_{H_t(\widehat{\theta}_t)^{-1}}$ by $\sqrt{\kappa} \lVert z_t \rVert_{V_t^{-1}}$. Finally, applying Elliptic Potential Lemma (Lemma~\ref{lemma:elliptic-potential-lemma}) finishes the proof. Details are in appendix~\ref{appendix:logistic}.\hfill$\qed$
\end{proof}

\begin{remark}
\label{remark:comparison-with-mehta}
To the best of our knowledge,~\cite{mehta2023sample} is the only work similar to ours with theoretical guarantees. Therefore, we highlight in detail the major differences. First, the algorithm design is entirely different as we choose both the actions for any context by maximizing the exploration bonus $b_t(x,a,a')$, while \cite{mehta2023sample} chooses one action uniformly at random. This can be wasteful in practice as the choice of the second action is equally crucial in preference-based learning. Further, the context selection rule is entirely different. While we pick the context with the highest exploration bonus (Eq.~\ref{eq:choose-x}),~\cite{mehta2023sample} uses an uncertainty estimate calculated via upper and lower confidence bounds of the rewards. Moreover, their suboptimality gap scales linearly with $\kappa$, while our guarantee only scales as $\sqrt{\kappa}$.

Next,~\cite{mehta2023sample} competes against the \emph{Borda} winner, while we do so against the (stronger) \emph{Condorcet} winner. Moreover, ~\cite{mehta2023sample} assumes that the Borda function $g^*(x,a) \!=\! \EE_{a'\sim \texttt{Unif}(\cA)}[\sigma(r^*(x,a) \!-\! r^*(x,a'))]$ lie in the same function space as the reward function $r^*(x,a)= \inner{\theta^*}{\phi(x,a)}$. This assumption doesn't hold in general due to the non-linearity in $\sigma$ and hence restricts the preference probabilities. The assumption holds trivially if each $\phi(x,a)$ is a one-hot vector $\mathbf{e}_{x,a}$, but it pushes $\theta^*$ to $\lvert\cX\rvert\cdot\lvert\cA\rvert$ dimensions and blows up the suboptimality gap significantly. We remove this restriction and provide the guarantee in dimension $d \ll \lvert\cX\rvert\cdot\lvert\cA\rvert$, by crucially exploiting properties of the sigmoid function.
\end{remark}
  
\begin{remark}[Extension to Direct Preference Optimization (DPO)]
  Another popular alignment algorithm DPO~\cite{rafailov2023direct} does not train a reward model separately, rather it uses the log-probability $r_\theta(x,a)= \log\pi_\theta(a|x)-\log\pi_{\text{ref}}(a|x)$ as the reward, where $\theta \in \mathbb{R}^d$ parameterizes the policy to be learnt. For example, Softmax policies take the form $ \pi_\theta(a|x) \propto \exp(f_\theta(x,a))$, where $f_\theta$ is a differentiable function. Now if we assume $f_\theta$ to be linear, then $\pi_\theta$ becomes a log-linear policy, i.e., $\log \pi_\theta(a|x) \propto \langle\theta,\phi(x,a)\rangle$, which eventually makes $r_\theta$ a (shifted) linear function. Hence, one would be able to apply our proposed approach to learn the policy parameter $\theta$ directly from the preference dataset $\cD$.  
\end{remark}

\subsection{Generalization Beyond BTL Model}
\label{subsection:general-function-approximation}
In this section, we remove the assumption of the BTL preference model and assume access to a non-parametric function class 
\begin{align*}
 \cF \!=\! \{f:\cX \!\times\! \cA \! \times\! \cA\rightarrow [0,1]:  f(x,a,a') \!+\! f(x,a',a) \!=\! 1\}~,   
\end{align*}
where $f(x,a,a')$ denotes the probability that action $a$ wins over action $a'$ (denoted by $a \succcurlyeq a'$) given context $x$ when the preference function is $f$, i.e., $f(x,a,a') = \PP[a \succcurlyeq a'|x,f]$. 
Note that in this case, there is no latent reward model, and hence, this is a strict generalization of the BTL model. Now, we assume that there is a realizable $f^* \in \cF$ from which preferences are observed at each round $t$, i.e. $y_t \sim \texttt{Ber}(f^*(x_t,a_t,a'_t))$. Further, we assume a \textit{Condorcet} winner at each context $x \in \cX$ w.r.t. $f^*$, i.e. there is an action $a^*(x) \in \cA$ such that $f^*(x,a^*(x), a) \geq \nicefrac{1}{2}$ for all $a \in \cA$. Accordingly, the sub-optimality gap of policy $\pi_T$ is defined as
\begin{align*}
    R(T) = \max\nolimits_{x \in \cX} f^*(x, a^*(x), \pi_T(x)) - \nicefrac{1}{2}~.
\end{align*}
In this setting, we propose a generalization of \texttt{APO}, namely, \texttt{APO-Gen} (Algorithm~\ref{algo:act-con-sel-gen-func-approx} in Appendix~\ref{section:general-function-approximation}). Similar to \texttt{APO}, it selects a context and a pair of actions at each round by maximizing an uncertainty score that depends on the complexity of the function class $\cF$. However, unlike \texttt{APO}, (a) at each round, it prunes out sub-optimal actions for every context, and (b) after $T$ rounds, the final policy is to sample an action uniformly from the remaining near-optimal actions for each context.  \texttt{APO-Gen} enjoys the following guarantee (details and proof in Appendix~\ref{section:general-function-approximation}).

\begin{theorem}[Suboptimality Gap of \texttt{APO-Gen}]\label{thm:gen}
    Let $\delta \in (0,1)$, and $\cN(\cF)$ and $d_\cE(\cF)$ be the covering number and Eluder dimension of $\cF$ respectively. Then, with probability $ \geq 1-\delta$, the policy returned by \texttt{APO-Gen} enjoys sub-optimality gap
    \begin{align*}
      R(T) \leq \widetilde{O}\bigg(\sqrt{{\log(\cN(\cF)T/\delta) d_\cE(\cF)}/{T}}\bigg)~, 
    \end{align*}
\end{theorem}
For the BTL preference model, we have $\log \cN(\cF)=O(d\log T)$ and $d_\cE(\cF)=O(\kappa^2 d \log T)$. Hence, we get an $\Tilde{O}(\kappa d/\sqrt{T})$ sub-optimality gap for \texttt{APO-Gen}, which is $\sqrt{\kappa}$ factor loose than Theorem~\ref{theorem:logistic-case-regret-bound}. This is because we crucially use self-concordance of the sigmoid function in Theorem~\ref{theorem:logistic-case-regret-bound} to shave this extra $\sqrt{\kappa}$ factor. Nevertheless, this result is general enough to subsume other preference models (e.g., Thurstone) beyond the BTL model.
\section{Experiments}
\label{sec:exps}
We first present a practical version of \texttt{APO}, which largely follows the former with minor changes adapted for the computationally efficient implementation required in large-scale experiments. Next, we present experimental results that demonstrate the efficacy of \texttt{APO} over random sampling (hereafter denoted by \texttt{Random}) and baselines~\cite{mehta2023sample,muldrew2024active}. Hyperparameter details are given in Appendix~\ref{appendix:experiment-details}. The experiment code can be found \href{https://github.com/nirjhar-das/active-preference-optimization}{here}.
\begin{algorithm}[t]
\caption{\texttt{APO} (Practical version)}
\label{algo:apo-rlhf}
\begin{algorithmic}[1]
\REQUIRE Context-generation pairs $\cM \!=\! \{(x, a, a')\}$, sample budget $T$, encoder $\phi$, SFT policy $\pi_{\text{SFT}}$, log-loss $\cL$, batch size $B$, uncertainty regularizer $\lambda \!>\! 0$, KL regularizer $\beta \!>\!0$, learning rate $\eta \!>\!0$. Initialize $V_1 = \lambda I, \widehat{\theta}_1 =0, \cD = \emptyset $.
\FOR{batch $t=1, \ldots, \lfloor T/B\rfloor$}
\STATE Set $ b_t(x,a,a') = \matnorm{\phi(x,a) - \phi(x,a')}{V_t^{-1}}$ for each  $(x,a,a') \in \cM$. Set $\cM_t=\emptyset$.\alglinelabel{line:score-apo-rlhf}
\FOR{$j=1,\ldots,B$}
\STATE Pick $(x_{t,j},a_{t,j},a'_{t,j})=\!\!\!\argmax\limits_{(x,a,a') \in \cM \setminus \cM_t} \!\!\!b_t(x,a,a')$; Observe preference $y_{t,j}$.
\alglinelabel{line:pick-top-rlhf}
\STATE Update $\cM_t \gets \cM_t \cup \{(x_{t,j}, a_{t,j}, a'_{t,j})\}$ and $\cD \gets \cD \cup \{(x_{t,j},a_{t,j},a'_{t,j},y_{t,j})\}$.
\ENDFOR
\STATE Update $\widehat{\theta}_{t+1} \gets \texttt{Gradient-step}(\cL,\widehat{\theta}_t, \cD,\eta)$. 
\STATE Update $V_{t+1} \gets V_t + \sum_{j=1}^B z_{t,j} z_{t,j}^\top $, where $z_{t,j} = \phi(x_{t,j},a_{t,j}) - \phi(x_{t,j},a'_{t,j})$. \alglinelabel{line:Vt}
\ENDFOR
\STATE Set reward $\widehat r_T(x,a) = \phi(x,a)^\top \widehat{\theta}_{\lfloor T/B\rfloor +1} \forall (x,a)$ and policy $\pi_T \gets \texttt{PPO}(\pi_{\text{SFT}}, \widehat r_T, \beta)$.\alglinelabel{line:policy}\alglinelabel{line:reward}
\end{algorithmic}
\end{algorithm}
\subsection{Practical Version of APO for RLHF}
In this practical version of \texttt{APO} (Algorithm~\ref{algo:apo-rlhf}), we access preference data in batches instead of being fully online. At the start of each batch $t$, we first compute the uncertainty $b_t(x,a,a')$ of each triplet $(x,a,a')$ ( Step~\ref{line:score-apo-rlhf}). This is similar to Algorithm~\ref{algo:act-con-sel-logB} except that here we compute the norm w.r.t. the inverted sample covariance matrix of feature differences $V_t^{-1}$ instead of $H_t(\widehat\theta_t)^{-1}$ (Step~\ref{line:Vt}). We do so since it is both compute and memory efficient for large scale experiments. To maximize exploration, only those $B$ triplets $(x,a,a')$ are sent for labeling in a batch that have the highest uncertainty $b_t(x,a,a')$, and those are stored in a buffer $\cD$. At the end of each batch $t$, we update the parameter estimate $\widehat{\theta}_t$ via a black-box gradient-descend-based algorithm (e.g. Adam~\cite{KingBa15}) on the log-loss \eqref{eq:log-loss} over the dataset $\cD$. 
Finally, after the budget $T$ is exhausted, we first learn an estimate $\widehat r_T$ of the latent reward model $r^*$, and then align the policy via proximal policy optimization (\texttt{PPO})~\cite{schulman2017proximal}, which takes as input the SFT policy $\pi_{\text{SFT}}$, the learnt reward model $\widehat r_T$ and a KL-regularizer $\beta$, and returns $\pi_T$. 
\begin{figure}[!bt]
\centering
\includegraphics[width=0.3\textwidth]{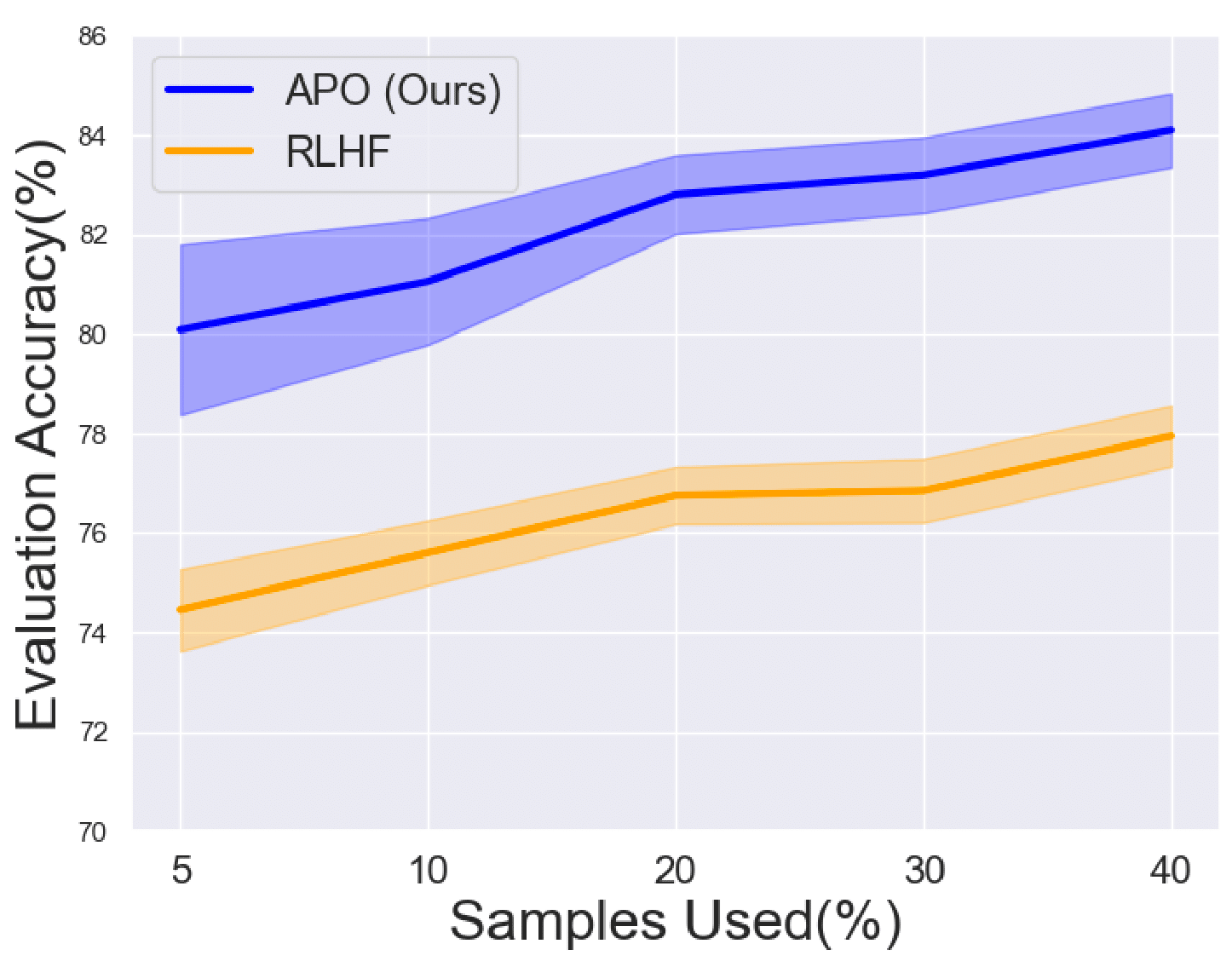}
\includegraphics[width=0.3\textwidth]{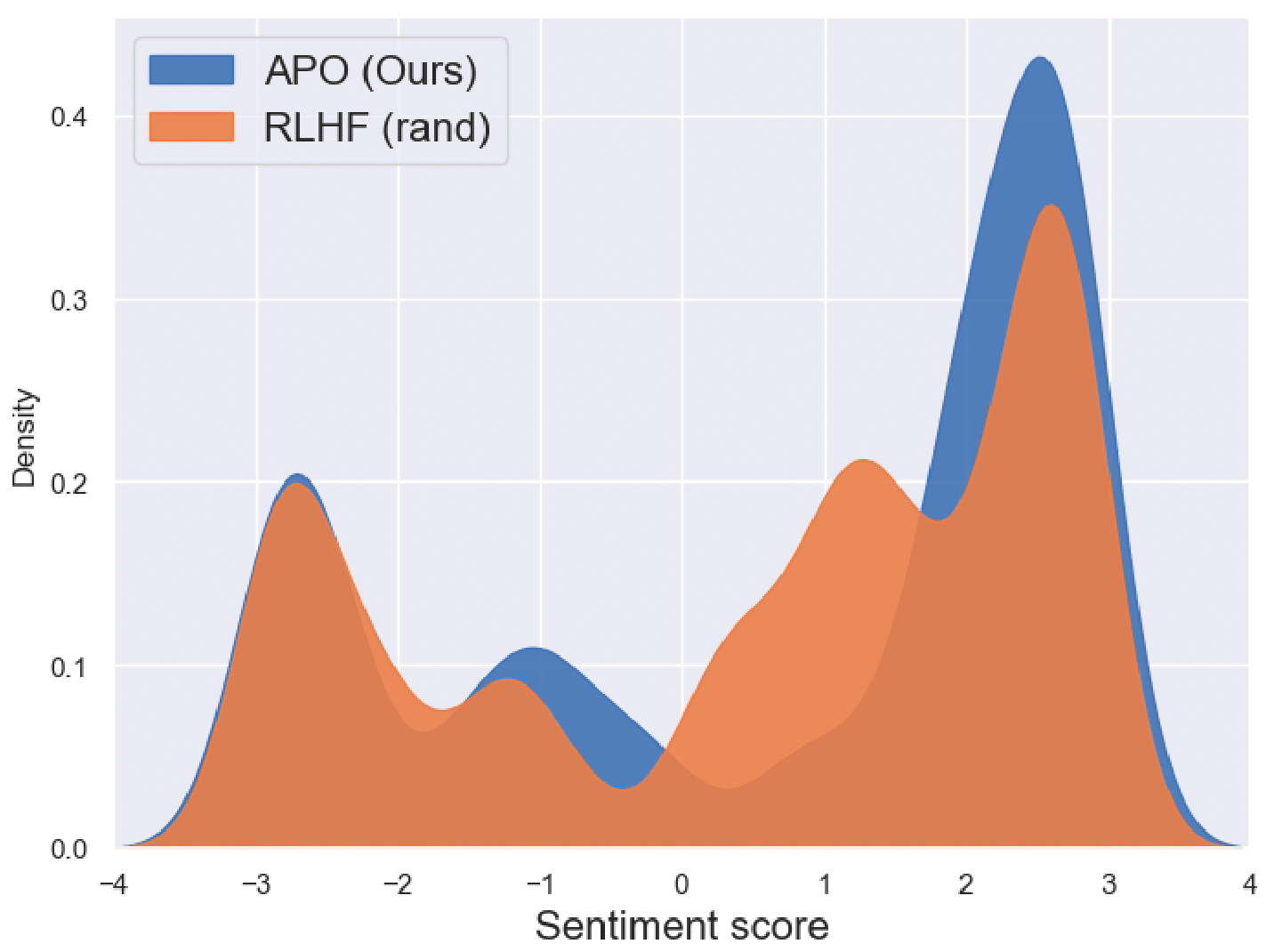}
\includegraphics[width=0.38\textwidth]{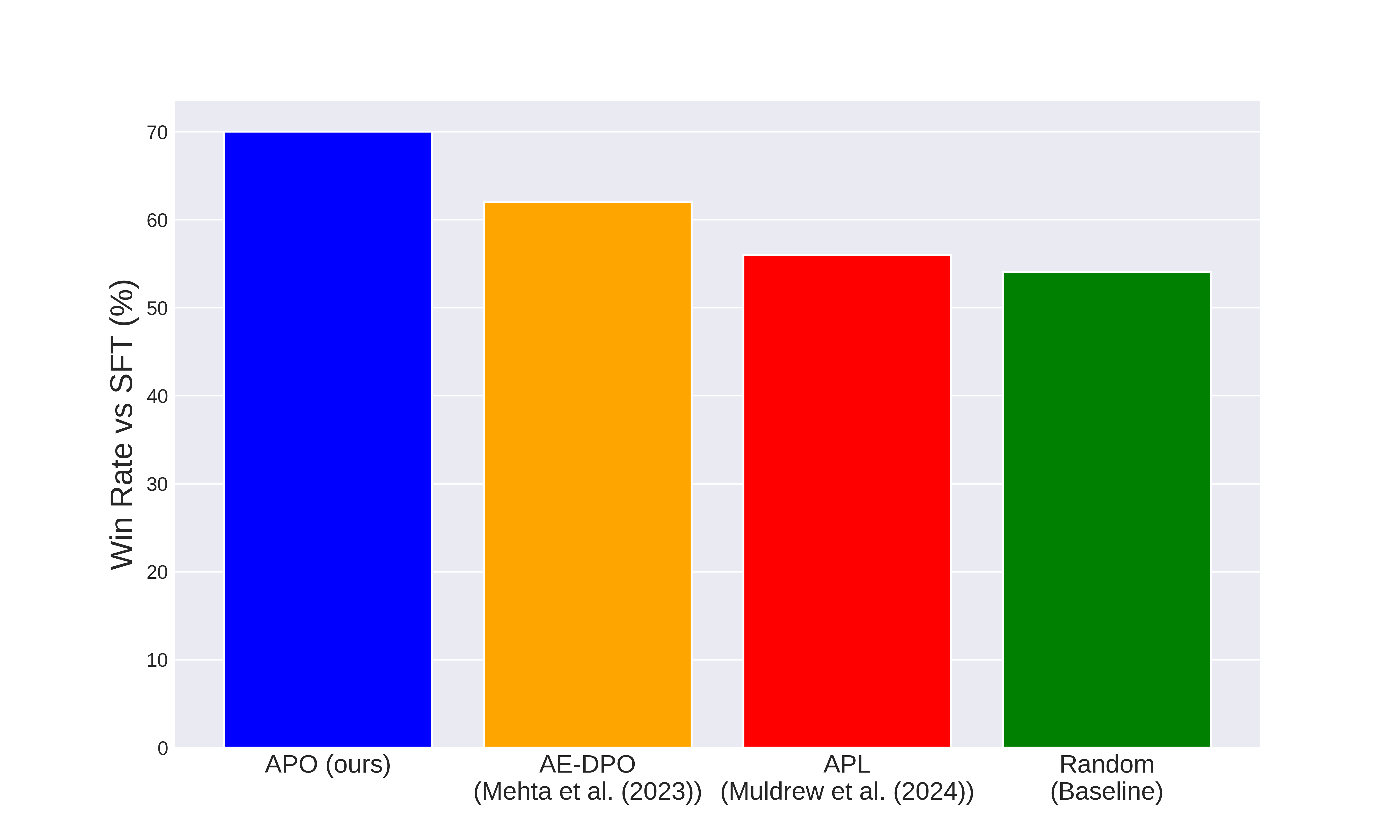}
\includegraphics[width=0.24\textwidth]{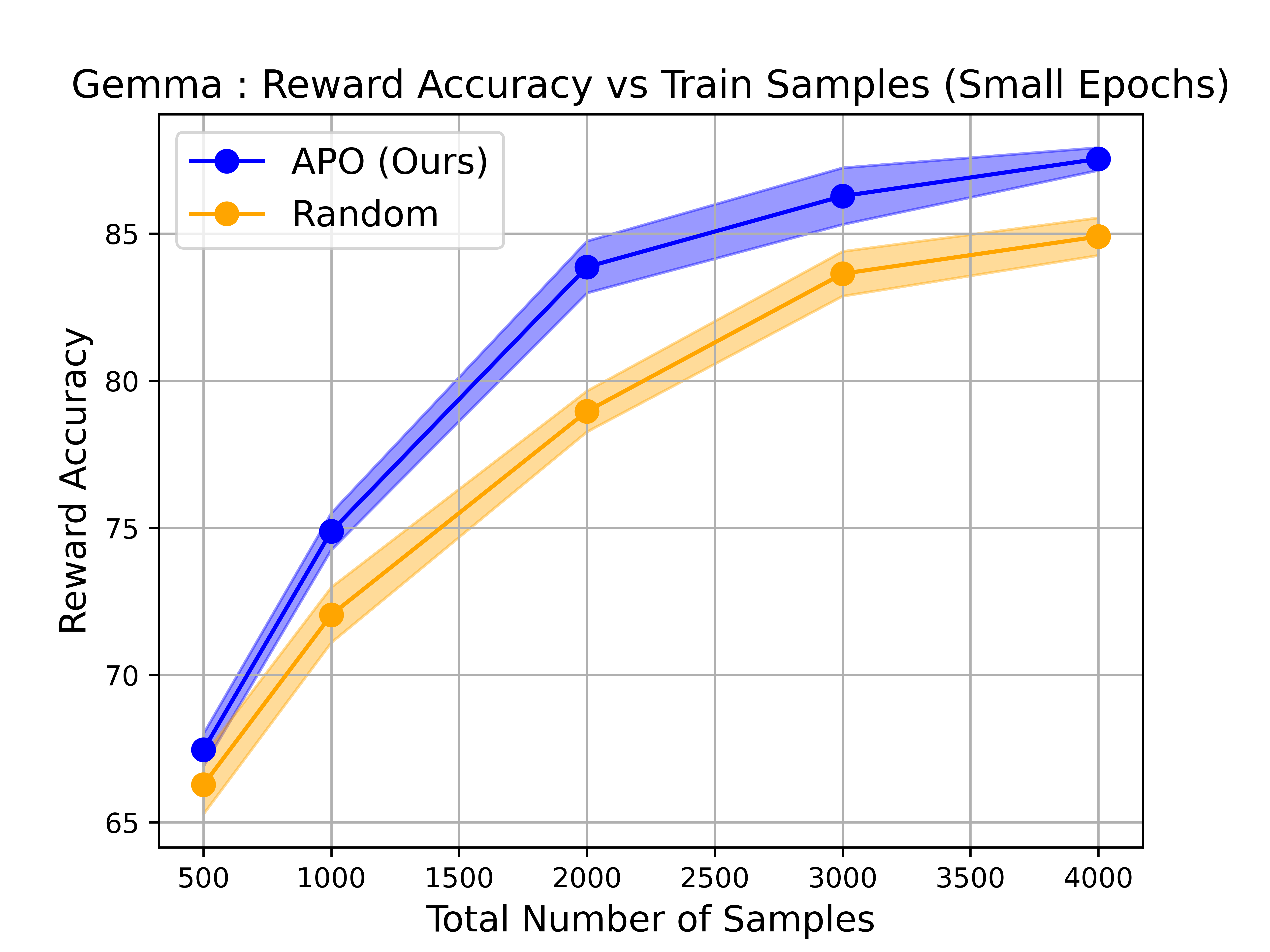}
\includegraphics[width=0.24\textwidth]{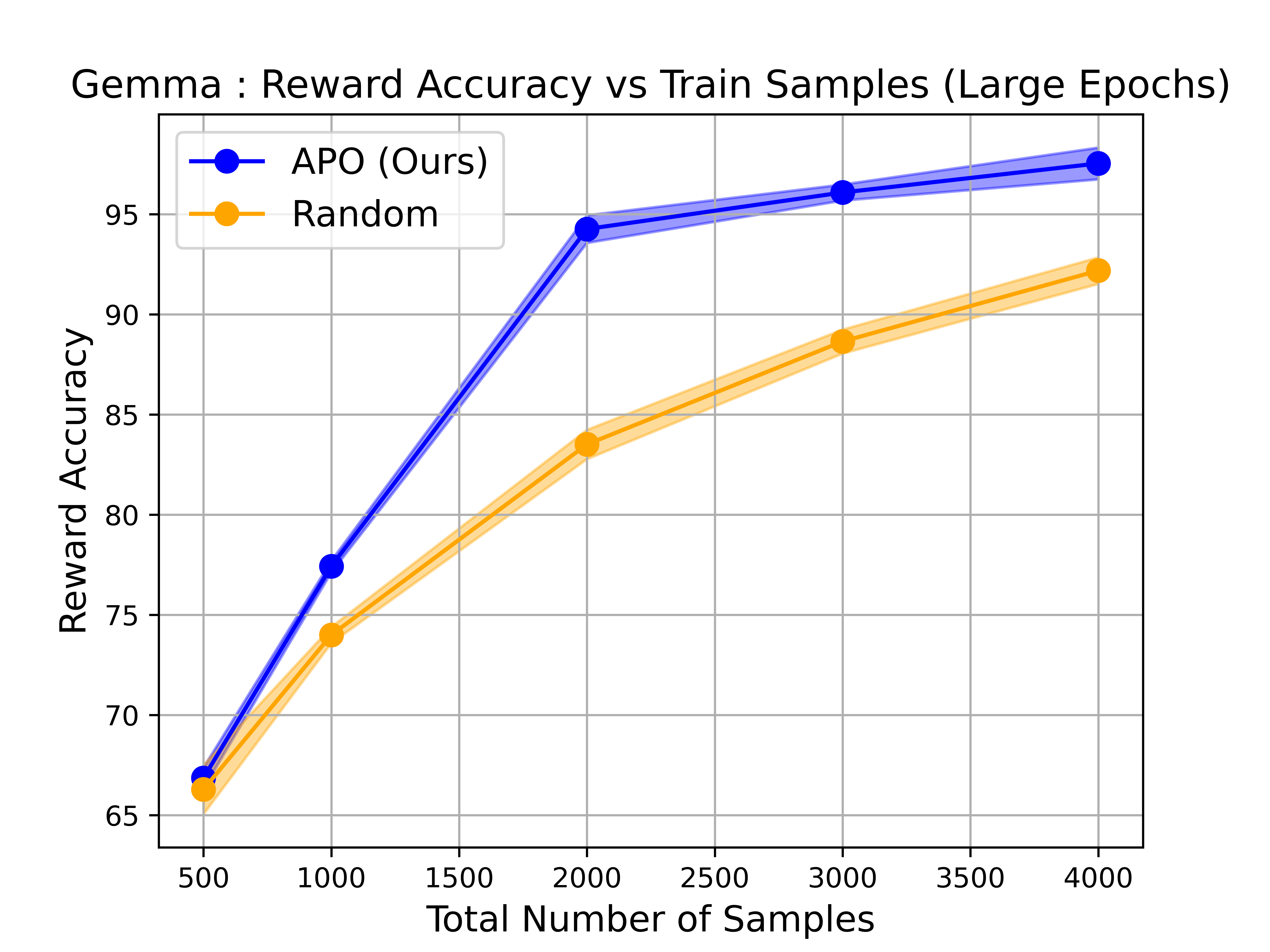}
\includegraphics[width=0.25\textwidth]{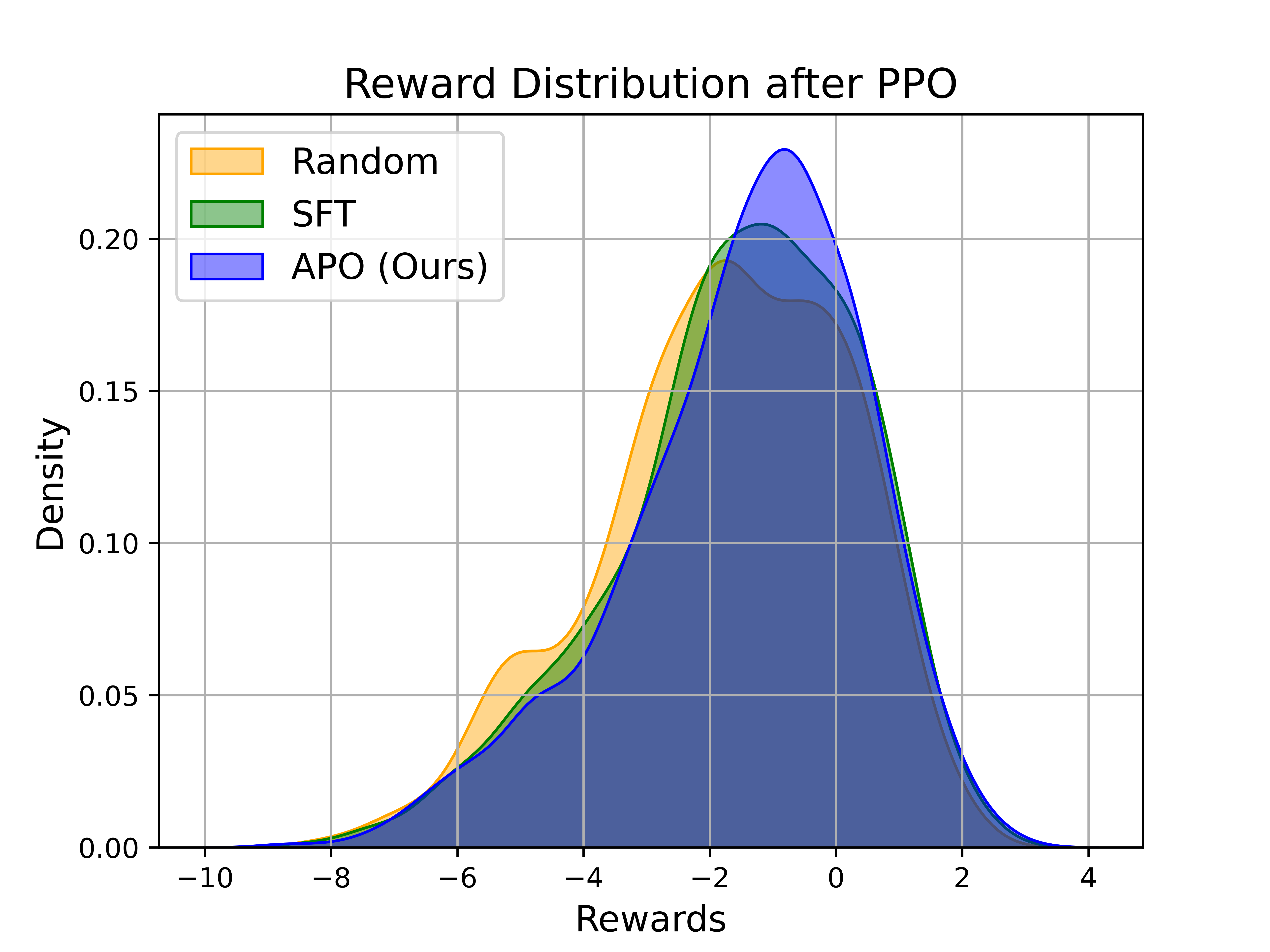}
\includegraphics[width=0.24\textwidth]{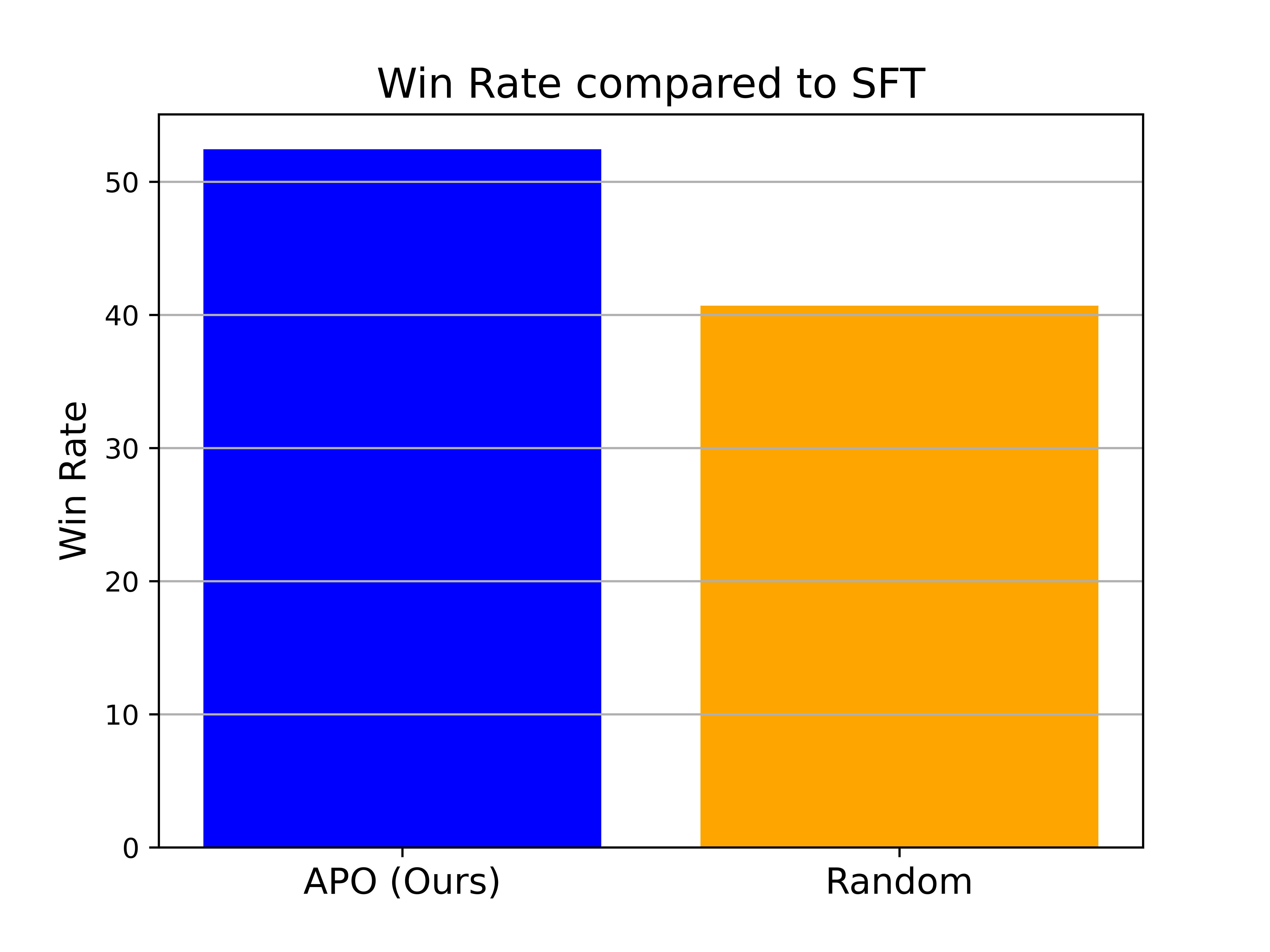}
\caption{\small{\textbf{Top Row: Controlled Sentiment Generation Task: Left:} Evaluation accuracy of trained reward model vs. no. of samples (in \%) comparing \texttt{APO} with \texttt{Random}. \textbf{Middle:} Sentiment score distribution of aligned policies trained on reward model learned with \texttt{APO} and on \texttt{Random}'s highest accuracy reward model. Generations by \texttt{APO}-trained reward is more shifted towards positive, showing better alignment than \texttt{Random}. \textbf{Right:} Win rates of \texttt{APO}, \texttt{AE-DPO}~\cite{mehta2023sample} and \texttt{APL}~\cite{muldrew2024active} and \texttt{Random} against SFT policy. \texttt{APO} outperforms \texttt{AE-DPO}, \texttt{APL} and \texttt{Random} by $72 : 62 : 56 : 54$ win rate.
\textbf{Bottom Row: Single-turn Dialogue Task: Left and 2nd Left:} Evaluation accuracy of trained reward model vs. no. of samples comparing \texttt{APO} with \texttt{Random}, when the number of epochs is 5 (\textbf{Left}) and 20 (\textbf{2nd Left}). Evaluation accuracy of \texttt{APO} is higher than the \texttt{Random} in both cases. \textbf{2nd Right:} Reward distribution of \texttt{APO}-aligned, SFT and \texttt{Random}-aligned policies for generations on prompts in the test dataset. Clearly, \texttt{APO}'s alignment is better than \texttt{Random}.
\textbf{Right:} Win rates of \texttt{APO} and \texttt{Random} aligned policies against SFT policy. \texttt{APO} outperforms \texttt{Random} by $55 : 40$ win rate.}}
    \label{fig:rlhf_vs_rand}
\end{figure}
\subsection{Results on Controlled Sentiment Generation Task} 
In this experiment, we consider a user group that prefers positive sentiment completions for movie reviews in the IMDb dataset~\cite{imdb_dataset}. The goal is to output generations $a$ that exhibit positive sentiment, catering to the user group's preferences for a given context $x$. For controlled evaluation, we generate preference pairs $(a,a')$ utilizing a pre-trained sentiment classifier where $\mathbb{P}(\text{positive-sentiment}\mid x,a)>\mathbb{P}(\text{positive-sentiment}\mid x,a')$. We generate a total of 10000 preference samples $(x,a,a')$ and use 4:1 train-test split.
For the SFT policy, we fine-tune GPT-2~\cite{radford2019language} on preferred reviews from the train set (8000 samples) and use this GPT-2 backbone for both reward learning and policy alignment. 

\textbf{For reward training}, we adaptively select context and generation pairs from the train set. We use the feature representation $\phi(x,a)$, and estimate the uncertainty $b_t(x,a,a')$ for each $(x,a,a')$ in the train set and select top-$B$ samples to update the reward model. We repeat this process $K$ times and return the final trained reward model. We evaluate the performance of the trained reward model against \texttt{Random} (where we select $B$ samples randomly at every batch)
on the test set of 2000 samples. Figure~\ref{fig:rlhf_vs_rand} (\textbf{Top Left}) shows the result: evaluation accuracy of the reward model learned by \texttt{APO} is much higher than the one learned via \texttt{Random} even when \texttt{APO}'s sample budget is only $5\%$ of the data and \texttt{Random}'s is $40\%$.

\textbf{For policy alignment}, we align the SFT policy with respective trained reward models (via \texttt{APO} and \texttt{Random}) using \texttt{PPO} (step~\ref{line:policy}). To demonstrate the effectiveness of adaptive sampling, we use the reward model trained on a sample budget of only $10\%$ for \texttt{APO}, while we use the highest accuracy reward model (corresponding to $40\%$ samples) for \texttt{Random}. Similar to \cite{rafailov2023direct}, the generations of aligned policies are evaluated against the ground truth reward $r^*$ for positive sentiment, which is provided by the pre-trained sentiment classifier. From Figure~\ref{fig:rlhf_vs_rand} (\textbf{Top Middle}), it can be seen that the reward distribution of the generations of \texttt{APO} -aligned policy achieves a higher density of positive sentiment compared to that aligned by \texttt{Random}. Moreover, from Figure~\ref{fig:rlhf_vs_rand} (\textbf{Top Right}) it is evident that \texttt{APO} outperforms \texttt{APL}~\cite{muldrew2024active}, \texttt{AE-DPO}~\cite{mehta2023sample} and \texttt{Random} in terms of win-rate against the SFT policy demonstrating the efficiency of the proposed method.

\subsection{Results on Single-turn Dialogue task}
In this experiment, we use Anthropic-HH~\cite{anthropic_dataset} preference dataset and instruction-tuned Gemma-2b~\cite{team2024gemma} language model. We collect all the contexts with single-turn dialogues and split these into two sets in a 4:1 ratio. We put samples from the larger collection into three buckets based on reward difference between \textit{chosen} and \textit{rejected} responses using \href{https://huggingface.co/Ray2333/reward-model-Mistral-7B-instruct-Unified-Feedback}{Mistral-7b} reward model as latent reward $r^*$. These buckets contain data points that are progressively easier to classify: (B1) reward difference between $-1$ to $1$, (B2) reward difference between $1$ to $3$ and (B3) reward difference of more than $3$. Out of these three buckets, we take 4500 samples from (B1), 2500 from (B2), and 1000 from (B3) to carefully curate a collection of 8000 training samples. Such a collection (more samples taken from the buckets with a smaller reward difference and fewer samples from the bucket with a larger reward difference) highlights the importance of selecting prompts carefully to obtain useful information during reward training --  randomly sampling contexts to collect feedback is more likely to hurt the performance in such a setting. For the test set, we sample $2000$ data points from the smaller collection set aside.

\textbf{Reward Evaluation.} We compare the reward models learnt by \texttt{APO} and \texttt{Random} by computing the $\%$ of samples in the test set for which the models assign higher reward to the \emph{chosen} responses than to the \textit{rejected} responses. We study how this accuracy changes with the number of batches or epochs, keeping the sample budget the same (Fig.~\ref{fig:rlhf_vs_rand} (\textbf{Bottom Left}) for 5 epochs and Fig.~\ref{fig:rlhf_vs_rand} (\textbf{Bottom 2nd Left}) for 20 epochs). We observe that \texttt{APO} always outperforms \texttt{Random}. We also see that the reward accuracy increases with an increasing number of epochs for a given sample budget. We show results by varying training samples till $4000$ as we want to demonstrate the effectiveness of \texttt{APO} under budget constraint.

\textbf{Win Rate.} Based on reward models learnt by \texttt{APO} and \texttt{Random}, we fine tune the SFT policy with \texttt{PPO} to obtain \texttt{APO}-policy and \texttt{Random}-policy respectively. Then we generate responses for contexts in the test set using \texttt{APO}, \texttt{Random}, and SFT policies, and get them evaluated by the Mistral-7b reward model. The reward distributions of these three policies are shown in Fig.~\ref{fig:rlhf_vs_rand} (\textbf{Bottom 2nd right}). It can be seen that \texttt{APO} has a higher density of positive rewards. Win rate of \texttt{APO}-policy and \texttt{Random}-policy against SFT policy is shown in Fig.~\ref{fig:rlhf_vs_rand} (\textbf{Bottom Right}), which shows that \texttt{APO} outperforms \texttt{Random} by a $55:40$ win rate.
\section{Conclusion}
In this work, we studied whether sampling prompts uniformly at random from a dataset to solicit feedback is sample efficient. We first showed that this method can suffer a constant suboptimality gap when aligning a language model policy with human preferences. Next, we characterized the sub-optimality gap lower bound for any active-learning algorithm with sample budget $T$ and problem dimension $d$, showing it to be $\Omega(d/\sqrt{T})$. Then, we proposed an algorithm \texttt{APO}, which actively samples prompts to achieve an $\Tilde{O}(d/\sqrt{T})$ sub-optimality gap. We also extended the results of \texttt{APO} to general function approximation to better capture modern-day RLHF training. Finally, we showed its efficacy over the random-sampling baseline on practical datasets. Although \texttt{APO}'s sub-optimality gap is minimax optimal in $d$ and $T$, the optimal dependence on $\kappa$ is unknown and seems to be an interesting future work.
\bibliographystyle{splncs04}
\bibliography{ref}

\newpage
\appendix

\newpage
{\centering  {\Large\bfseries Active Preference Optimization for Sample Efficient RLHF: Appendix \par}}
\section{Proof of Theorem~\ref{theorem:lower-bound}}
\label{appendix:uniform-lower-bound}
In this section, we restate Theorem~\ref{theorem:lower-bound} concerning the performance lower bound of a uniform learner (Definition~\ref{def:uniform-learner}) and provide a detailed proof by constructing a lower bound instance.
\begin{figure}[!ht]
\centering
\begin{tikzpicture}
    \draw[thick,->] (0,0) -- (-0.75,1.299) node[anchor=south] {$z_b$};
    \draw[thick,->,red] (0,0) -- (1.25,2.165) node[anchor=south] {$\theta^*$};
    \draw[thick,->,blue] (0,0) -- (2.5,0) node[anchor=north] {$\widehat{\theta}$};
    \draw[thick,->] (0,0) -- (1.5,0); 
    \node[align=left] at (1.5,-0.3) 
    {$z_g$};
    \draw[dashed,->,black] (0,0) -- (0,2.2);
    \draw[dashed,->,black] (0,0) -- (-1,0);
    \node[align=left] at (0,-0.3) {$(0,0)$};
\end{tikzpicture}
\caption{\small{Visualization of the instance for Theorem~\ref{theorem:lower-bound}. Here, $z_g$ represents feature difference vectors for \emph{good} contexts, and $z_b$ represents feature difference for the \emph{bad} context. $\theta^*$ and $\widehat{\theta}$ are the true and learnt parameters respectively.}}
\label{fig:lower-bound-instance}
\end{figure}
\begin{theorem}[Lower bound of uniform sampling strategy]
     There exists a problem instance $(\cX, \cA, \theta^*)$ for which the policy learnt by a Uniform Learner \texttt{Alg} under the budget $T \ll \lvert \cX \rvert$ suffers $\Omega(1)$ sub-optimality gap with high probability.
\end{theorem}
\begin{proof}
Let the number of contexts be $\abs{\cX}=N$. Assume $T \ll N$. Otherwise, if the sample budget $T > N$, then one can just collect data for every context, and the setting becomes trivial.
We divide $\cX$ into two disjoint subsets: a \emph{good} set $\cX_g$ and a \emph{bad} set $\cX_b$. We assume w.l.o.g. that $\abs{\cX_b} =1$, and we denote the \emph{bad} context by $b$. Let the action set be $\cA = \{a,a'\}$ for all contexts, and let $a$ have a higher reward than $a'$. Let $\phi:\cX \times \cA \rightarrow \RR^2$ be a feature map and $z_x = \phi(x,a) - \phi(x,a')$ be the feature difference vector at context $x$. Fix an $\alpha >0$ and
consider the problem instance:
\begin{align*}
 \theta^* = \alpha\begin{bmatrix} \frac{1}{2} & \frac{\sqrt{3}}{2}\end{bmatrix}^\top,\quad z_{b} = \begin{bmatrix} -\frac{1}{2} & \frac{\sqrt{3}}{2}\end{bmatrix}^\top,\quad z_x = \begin{bmatrix} 1 & 0\end{bmatrix}^\top,\, \forall\ x \in \cX_g~.  
\end{align*}
Note that $\matnorm{\theta^*}{2} = \alpha$. From this construction (see Fig.~\ref{fig:lower-bound-instance}), it is clear that both for \emph{good} and \emph{bad} contexts, $z_x^\top \theta^* = \alpha/2 >0$, which implies that indeed action $a$ has higher reward than $a'$. 

Let $\cE_1$ be the event that all the $T$ sampled contexts are \emph{good} (i.e., from $\cX_g$). Since, under uniform sampling, for a random context $X$, $\PP[X \in \cX_g] = 1 - 1/N$, we have $\PP[\cE_1] = (1 -1/N)^T$. Let $\cE_2$ be the event that all observed preferences $y_1, \ldots y_T$ are equal to $1$. Since, for a random preference $Y$ given a context $x$, $\PP[Y = 1 | x] = \sigma(\alpha/2)$, we have $\PP[\cE_2|\cE_1]=  \sigma(\alpha/2)^T$.

Now, under the event $\cE_1 \cap \cE_2$, the MLE $\widehat{\theta}$ constrained to the same norm as $\theta^*$ is given by
\begin{align*}
    \widehat{\theta} = \argmin_{\theta\in \RR^2:\matnorm{\theta}{2} \leq \alpha} \sum\nolimits_{t=1}^T \log\rbr{1 + e^{-\alpha \theta_1}}= \argmin_{\theta\in \RR^2:\matnorm{\theta}{2} \leq \alpha}  \log\rbr{1 + e^{-\alpha \theta_1}}~.
\end{align*}
It is easy to see that $\widehat{\theta} = \begin{bmatrix}\alpha & 0\end{bmatrix}^\top$. For any $x \in \cX_g$, the predicted reward difference between actions $a$ and $a'$ is $z_x^\top \widehat{\theta} = \alpha > 0$. Thus, $\widehat{\theta}$ predicts the better action $a$ correctly for all \emph{good} contexts $\cX_g$. However, for context $b$, the 
reward difference is $z_b^\top {\widehat{\theta}} = -\frac{\alpha}{2} < 0$. Thus, $\widehat \theta$ wrongly predicts $a'$ as the better action for the \emph{bad} context $b$. This yields a constant sub-optimality gap
\begin{align*}
    R(T,b) = \phi(b,a)^\top\theta^* - \phi(b,a')^\top\theta^* = \alpha/2 = \Omega(1).
\end{align*}
Finally, it remains to show that the event $\cE_1 \cap \cE_2$ happens with high probability. To this end, we choose $\alpha = 2\log(N - 1)$ which yields $\sigma(\alpha/2) = 1 - 1/N$. This yields
\begin{align*}
    \PP[\cE_1 \cap \cE_2]= \PP[\cE_2 | \cE_1] \PP[\cE_1]
    = \Big(1 - \frac{1}{N}\Big)^{2T}\geq 1 - \frac{2T}{N}~. 
\end{align*}
The last step uses $T \ll N$, which completes the proof.
\end{proof}
\section{Proof of Theorem~\ref{theorem:active-learning-lower-bound}}
\label{appendix:active-lower-bound}

We first re-state the theorem and then give its proof. The proof is based on reducing the contextual preference bandit instance to a logistic bandit instance. However, note that the regret lower bound of the logistic bandit does not help us in this case because our goal is to lower bound the sub-optimality gap and not regret. In this section, we use $D_{KL}(P,Q)$ to denote the KL-divergence between two probability distributions $P$ and $Q$.

\begin{theorem}[Lower Bound for any sampling strategy]
    Let $\Theta = \{-\frac{1}{\sqrt{T}},\frac{1}{\sqrt{T}}\}^d $, $\cX$ be a finite set of contexts, $\cA = \{-\frac{1}{2}, \frac{1}{2}\}^d$. Then, for any algorithm, there exists a parameter $\theta^* \!\in\! \Theta$ such that sub-optimality gap of a policy learnt by the algorithm after collecting $T$ samples satisfies 
\begin{align*}
    \mathbb{E}_{\theta^*}\left[R(T) \right] \ge \Omega\left(\frac{d}{\sqrt{T}}\right)~,
\end{align*}
where the expectation is over the randomness of $(x_1,a_1,a'_1,y_1,\ldots,x_T,a_T,a'_T,y_T)$.
\end{theorem}

\begin{proof}
    Let $\pi$ be the policy learnt by any given sampling strategy. Without loss of generality, fix a context $x \in \cX$ and denote by $\pi_T(x) \in \cA$ the action predicted by the learned policy. For a preference parameter $\theta$, denote by $\PP_\theta$ the distribution over $(x_1,a_1,a'_1,y_1,x_2,\ldots,x_T,a_T,a'_T,y_T)$. By Divergence Decomposition Lemma (Lemma~\ref{lemma:divergence-decomposition-lemma}), we have, for $\theta, \theta' \in \Theta$,
    \begin{align*}
        D_{KL}(\PP_\theta, \PP_{\theta'}) \leq \frac{1}{8}\EE_\theta \left[\sum_{t=1}^T \langle a_t - a'_t, \theta - \theta' \rangle^2 \right]
    \end{align*}
    Note that for any $x \in \mathcal{X}$ and any two actions $a,a' \in \mathcal{A}$, $(a - a') \in [-1, 1]^d$. Now, we define the event $$\cE_{\theta, i} = \{\text{sign}(\pi_{T,i}(x) \neq \text{sign}(\theta_i)\}~,$$
    where $\pi_{T,i}(x)$ and $\theta_i$ are the $i$-th coordinates of $\pi_T(x)$ and $\theta$ respectively. Further, let $\theta' \in \Theta$ be such that $\theta'_i = -\theta_i$ and for $j \neq i$, $\theta'_j = \theta_j$. Note that $\cE_{\theta',i} = \cE_{\theta,i}^c$. Thus, by Bretagnolle-Huber Inequality (Lemma~\ref{lemma:bretagnolle-huber-inequality}),
    \begin{align*}
        \PP_{\theta}[\cE_{\theta,i}] + \PP_{\theta'}[\cE_{\theta',i}] &\geq \frac{1}{2} \exp(-D_{KL}(\PP_\theta, \PP_{\theta'}))\\ &\geq \frac{1}{2} \exp\left(-\frac{1}{8} \EE_\theta \left[\sum_{t=1}^T \langle a_t - a'_t, \theta - \theta' \rangle^2 \right] \right) \\
        &\geq \frac{1}{2} \exp\left(-\frac{1}{8} \sum_{t=1}^T {4}\theta_i^2 \right) = \frac{1}{2} \exp(-\frac{1}{2})~.
    \end{align*}
    Now, since we have $\lvert \Theta \rvert = 2^d$,
    \begin{align*}
        \sum_{\theta \in \Theta} \frac{1}{\lvert \Theta \rvert} \sum_{i=1}^d \PP_{\theta}[\cE_{\theta,i}] = \frac{1}{\lvert \Theta \rvert} \sum_{i=1}^d \sum_{\theta \in \Theta} p_{\theta, i} \geq \frac{2^{d-1}}{2^d} \cdot d \cdot \frac{1}{2} \exp(-\frac{1}{2}) = \frac{d}{4} \exp(-\frac{1}{2})~.
    \end{align*}
    Hence, there exists a $\theta^* \in \Theta$ such that $\sum_{i=1}^d \PP_{\theta^*}[\cE_{\theta^*,i}] \geq \frac{d}{4} \exp(-\frac{1}{2})$. Thus, the expected suboptimality gap is lower bounded as
    \begin{align*}
        \EE_{\theta^*}[R(T)] & \geq \EE_{\theta ^*}\left[\max_{a \in \mathcal{A}} \langle a - \pi_T(x), \theta^* \rangle \right] \\
        &= \EE_{\theta^*} \left[ \sum_{i=1}^d \lvert \theta^*_i \rvert \cdot \mathds{1}\{ \text{sign}(\pi_{T,i}(x)) \neq \text{sign}(\theta^*_i) \} \right] \tag{Best action has coordinate-wise same sign as $\theta^*$}\\
        &= \frac{1}{\sqrt{T}} \sum_{i=1}^d \EE_{\theta^*} \left[ \mathds{1}\{ \text{sign}(\pi_{T,i}(x)) \neq \text{sign}(\theta^*_i) \} \right] \\
        &= \frac{1}{\sqrt{T}} \sum_{i=1}^d \PP_{\theta^*}[\cE_{\theta^*,i}] \geq \frac{1}{\sqrt{T}} \cdot \frac{d \exp(-\nicefrac{1}{2})}{4}~,
    \end{align*}
    which completes the proof.
\end{proof}

\begin{lemma}[Divergence Decomposition]
\label{lemma:divergence-decomposition-lemma}
    Consider a logistic bandit instance with action space $\mathcal{A} \subset \RR^d$ and parameter $\theta \in \RR^d$. A policy $\pi$, over $T$ rounds, takes actions $a_t \in \mathcal{A}$ and observes reward $r_t \sim \texttt{Ber}(\sigma(\langle a_t, \theta\rangle))$, for all $t \in [T]$. Let the sequence of random actions and observations be denoted by $\mathcal{O} = (a_1, r_1, a_2, \ldots, a_T, r_T)$.  Further, let $\PP_\theta$ and $\PP_{\theta'}$ be the two distributions of $\mathcal{O}$ under parameters $\theta$ and $\theta'$. Then,
    \begin{align*}
        D_{KL}(\PP_\theta, \PP_{\theta'}) \leq \frac{1}{8}\EE_\theta \left[\sum_{t=1}^T \langle a_t, \theta - \theta' \rangle^2 \right]~.
    \end{align*}
\end{lemma}

\begin{proof}
    This lemma is a direct adaptation of~\cite[Lemma 15.1, Exercise 15.8(b)]{lattimore2020bandit} to the logistic bandit setting. From these results, we directly have
    \begin{align}
    \label{eq:total-kl-divergence-upto-T}
        D_{KL}(\PP_\theta, \PP_{\theta'}) = \EE_\theta \left[ \sum_{t=1}^T D_{KL}(\texttt{Ber}(\sigma(\langle a_t, \theta \rangle)), \texttt{Ber}(\sigma(\langle a_t, \theta' \rangle)) \right]~.
    \end{align}
    Now, upon simplification of the KL-divergence between two logistic distributions, we have,
    \begin{align}
    \label{eq:kl-divergence-at-t}
        &D_{KL}(\texttt{Ber}(\sigma(\langle a_t, \theta \rangle)), \texttt{Ber}(\sigma(\langle a_t, \theta' \rangle)) \nonumber \\
        &= (1 - \sigma(\langle a_t, \theta \rangle)) \langle a_t, \theta' - \theta \rangle + \log(\sigma(\langle a_t, \theta \rangle)) - \log(\sigma(\langle a_t, \theta' \rangle))~.
    \end{align}
    Let $p = \langle a_t, \theta \rangle$ and $q = \langle a_t, \theta' \rangle$. Now, via a Taylor expansion, we have,
    \begin{align*}
        \log(\sigma(q)) &= \log(\sigma(p)) + \frac{\dot{\sigma}(p)}{\sigma(p)} (q - p) + \frac{1}{2} \frac{d}{dp}\left(\frac{\dot{\sigma}(p)}{\sigma(p)}\right) \bigg|_{u \in (p,q)} (q - p)^2 \\
        &= \log(\sigma(p)) + (1 - \sigma(p))(q-p) - \frac{1}{2} \dot{\sigma}(u) (q - p)^2~. \tag{$\dot{\sigma}(p) = \sigma(p)(1 - \sigma(p))$}
    \end{align*}
    Therefore, substituting the Taylor expansion in~\ref{eq:kl-divergence-at-t}, we get,
    \begin{align*}
        D_{KL}(\texttt{Ber}(\sigma(\langle a_t, \theta \rangle)), \texttt{Ber}(\sigma(\langle a_t, \theta' \rangle)) = \frac{1}{2} \dot{\sigma}(u) \langle a_t, \theta' - \theta \rangle^2 \leq \frac{1}{8}  \langle a_t, \theta' - \theta \rangle^2 \tag{$\dot{\sigma}(u) \leq \frac{1}{4}$}~.
    \end{align*}
    Substituting this in~\ref{eq:total-kl-divergence-upto-T} gives the desired result.
\end{proof}

\begin{lemma}[Bretagnolle-Huber Inequality, Theorem 14.2 of~\cite{lattimore2020bandit}] 
\label{lemma:bretagnolle-huber-inequality}
Let $(\Omega, \mathcal{F}, P)$ and $(\Omega, \mathcal{F}, Q)$ be two probability spaces. Let $A \in \mathcal{F}$ and $A^c$ be its complement event. Then,
\begin{align*}
    \PP_P[A] + \PP_Q[A^c] \geq \frac{1}{2} \exp(- D_{KL}(P,Q))
\end{align*}
\end{lemma}
\section{Proof of Theorem~\ref{theorem:logistic-case-regret-bound}}
\label{appendix:logistic}
First, we state the result from the logistic bandit literature that characterizes the confidence set for the constrained maximum likelihood estimator. Here, we give one version of the confidence set from~\cite{lee2024improved}, but note that similar guarantees are also derived in~\cite{abeille2021instance}.
\begin{lemma}[Confidence Set for MLE (Theorem 1 of~\cite{lee2024improved}]
\label{lemma:MLE-confidence-set}
    Let $\widehat{\theta}_t$ be the constrained maximum likelihood estimator after $t-1$ time steps defined as follows:
\begin{align*}
    \widehat{\theta}_{t} = \argmin_{\theta \in \Theta}\left\lbrace -\sum_{s=1}^{t-1}  y_s \log(\sigma(z_s^\intercal \theta)) + (1 - y_s) \log(1-\sigma(z_s^\intercal \theta))\right\rbrace~.
\end{align*}
    Now define the set $$\cC_t(\delta) = \{ \theta \in \Theta: \cL_t(\theta) - \cL_(\widehat{\theta}) \leq \beta_t(\delta)^2 \}$$ where $\beta_t(\delta) = \sqrt{10d\log\left(\frac{St}{4d}+e\right) + 2(e-2+S)\log\left(\frac{1}{\delta}\right)}$. Then we have
    $P(\forall t \geq 1, \theta^* \in \cC_t(\delta)) \geq 1 - \delta$.
\end{lemma}
The details of the proof can be found in section 3.1 of~\cite{lee2024improved}. Next, we present another lemma that quantifies the parameter estimation error. Using this lemma and a novel self-concordance property, we will prove Lemma~\ref{lemma:bound-on-theta-hat-minus-theta-star-in-theta-hat-norm}.

\begin{lemma}[Lemma 6 of~\cite{lee2024improved}]
\label{lemma:bound-on-theta-hat-minus-theta-star-in-theta-star-norm}
    Let $\widehat{\theta}_t$ be defined above. Further, let $\theta^* \in C_t(\delta)$. Then, $$\lVert \widehat{\theta}_t - \theta^* \rVert_{H_t(\theta^*)}^2 \leq \gamma_t(\delta)^2 \coloneqq 2(2+2S) f(d,S,t,\delta)$$ where $$f(d,S,t, \delta) \coloneqq 2(e-2)(2+2S)d\log(\frac{5St}{d}) + 2(e-2)(2+2S)\log(\frac{t}{\delta}) + \frac{5d}{4} + \frac{d^2}{16St}$$
    Simplifying the RHS above, we have, $\gamma_t(\delta)^2 = CS^2\left(d\log\frac{St}{d} + \log\frac{t}{\delta}\right)$ for some $C>0$.
\end{lemma}
The proof of the lemma can be found in appendix C.4.4 of~\cite{lee2024improved}. Now we present the proof of lemma~\ref{lemma:bound-on-theta-hat-minus-theta-star-in-theta-hat-norm}.
\begin{lemma}
\label{lemma:bound-on-theta-hat-minus-theta-star-in-theta-hat-norm}
    Suppose $\theta^* \in \cC_t(\delta)$. Then, $\lVert \theta^* - \widehat{\theta}_t \rVert_{H_t(\widehat{\theta}_t)} \leq CS^{\nicefrac{1}{2}}\gamma_t(\delta)$.
\end{lemma}

\begin{proof} By Taylor's theorem, we have,
\begin{align*}
    &\cL_t(\widehat{\theta}_t) - \cL_t(\theta^*)\\ &= \nabla \cL_t(\theta^*)^\intercal (\widehat{\theta}_t - \theta^*) + \int_{v=0}^1 (1-v) (\widehat{\theta}_t-\theta^*)^\intercal \nabla^2 \cL_t(\theta^*) (\widehat{\theta}_t-\theta^*) dv\\
    &= \nabla \cL_t(\theta^*)^\intercal (\widehat{\theta}_t - \theta^*) + \sum_{s=1}^{t-1} \left[\int_{v=0}^1 (1-v)\dot{\sigma}(z_s^\intercal\theta^* + v(z_s^\intercal \widehat{\theta}_t - z_s^\intercal \theta^*))dv \right] (z_s^\intercal(\widehat{\theta}_t - \theta^*))^2\\
    &= \nabla \cL_t(\theta^*)^\intercal (\widehat{\theta}_t - \theta^*) + \lVert \widehat{\theta}_t - \theta^* \rVert^2_{\tilde{G}_t(\theta^*, \widehat{\theta}_t)} - \lambda \lVert \widehat{\theta}_t - \theta^* \rVert^2 
\end{align*}
    where we define $\tilde{G}_t(\theta^*, \widehat{\theta}_t) = \lambda \mathbf{I}_d  +  \sum_{s=1}^{t-1} \left[\int_{v=0}^1 (1-v)\dot{\sigma}(z_s^\intercal\theta^* + v(z_s^\intercal \widehat{\theta}_t - z_s^\intercal \theta^*))dv\right]z_s z_s^\intercal$. Thus, we obtain,
\begin{align*}
    \lVert \widehat{\theta}_t - \theta^* \rVert^2_{\tilde{G}_t(\theta^*, \widehat{\theta}_t)} &= \cL_t(\theta^*) - \cL_t(\widehat{\theta}_t) + \nabla \cL_t(\theta^*)^\intercal (\widehat{\theta}_t - \theta^*) + \lambda \lVert \widehat{\theta}_t - \theta^* \rVert^2
\end{align*}
Now, from a novel self-concordant analysis (see lemma~\ref{lemma:novel-self-concordance-analysis}), $H_t(\widehat{\theta}_t) \preccurlyeq C(2 + 2S)^2 \tilde{G}_t(\theta^*, \widehat{\theta}_t)$ for some $C>1.01$. Thus,
\begin{align}
    \lVert \widehat{\theta}_t - \theta^* \rVert^2_{H_t(\widehat{\theta}_t)} &\leq C(2 + 2S)^2 \lVert \widehat{\theta}_t - \theta^* \rVert^2_{\tilde{G}_t(\theta^*, \widehat{\theta}_t)} \nonumber \\
    &= C(2 + 2S)^2 \left[ \cL_t(\theta^*) - \cL_t(\widehat{\theta}_t) +  \nabla \cL_t(\theta^*)^\intercal (\widehat{\theta}_t - \theta^*) + \lambda \lVert \widehat{\theta}_t - \theta^* \rVert^2 \right] \nonumber \\
    &\leq C(2+2S)^2 \left[ 4\lambda S^2 + \beta_t(\delta)^2 + \nabla \cL_t(\theta^*)^\intercal (\widehat{\theta}_t - \theta^*) \right] \label{eq:RHS-of-theta-hat-theta*-bound} 
\end{align}
where the last inequality is because (a) $\widehat{\theta}_t, \theta^* \in \Theta$ which implies that $\lVert \theta^* - \widehat{\theta}_t \rVert \leq \text{diam}(\Theta) = 2S$ and (b) by lemma~\ref{lemma:MLE-confidence-set}, $\cL_t(\theta^*) - \cL_t(\widehat{\theta}_t) \leq \beta_t(\delta)^2$ since $\theta^* \in \cC_t(\delta)$ by assumption.

Thereafter, from the proof of Lemma 6 of~\cite{lee2024improved} it can be extracted that
$
    \lvert \nabla \cL_t(\theta^*)^\intercal (\widehat{\theta}_t - \theta^*)\rvert \leq \frac{\lVert \widehat{\theta}_t - \theta^* \rVert^2_{H_t(\theta^*)}}{2(2+2S)} + f(d,S,t,\delta)~.
$
Then using lemma~\ref{lemma:bound-on-theta-hat-minus-theta-star-in-theta-star-norm}, the R.H.S of~\ref{eq:RHS-of-theta-hat-theta*-bound} can be bounded by $2f(d,S,t,\delta)$. Thus, we now obtain,
\begin{align*}
    \lVert \widehat{\theta}_t - \theta^* \rVert^2_{H_t(\widehat{\theta}_t)} &\leq C(2+2S)^2 \left[ 4\lambda S^2 + \beta_t(\delta)^2 + 2f(d,S,t,\delta) \right]\\
    &\leq C(2+2S)^2 \left[\frac{1}{(2+2S)^2} + \beta_t(\delta)^2 + \frac{\gamma_t(\delta)^2}{2(2+2S)}\right] \tag{ $\lambda = \frac{1}{4S^2(2+2S)^2}$}\\
    &\leq C(2+2S)^2 \left[ \frac{1}{(2+2S)} + \beta_t(\delta) + \frac{\gamma_t(\delta)}{\sqrt{2(2+2S)}} \right]^2
\end{align*}
Taking square-root on both sides,
\begin{align*}
    \lVert \widehat{\theta}_t - \theta^* \rVert_{H_t(\widehat{\theta}_t)} &\leq C(2+2S)\left[\frac{1}{(2+2S)} + \beta_t(\delta) + \frac{\gamma_t(\delta)}{\sqrt{2(2+2S)}}\right] \\
    &= C(1+(2+2S)\beta_t(\delta) + \sqrt{2+2S}\gamma_t(\delta))\\
    &= CS^{3/2}\sqrt{\left(d\log(\frac{St}{d})+\log(\frac{t}{\delta})\right)}~,
\end{align*}
which proves the lemma.
\end{proof}
Now, we restate Theorem~\ref{theorem:logistic-case-regret-bound} and give its proof.
\begin{theorem}[Suboptimality Upper Bound]\label{thm:log-restated} 
Let $\delta \in (0,1)$. The suboptimality of the policy $\pi_T$ specified at the end of \texttt{APO} (algorithm~\ref{algo:act-con-sel-logB}) after running the algorithm for $T$ rounds is upper bounded with probability at least $1-\delta$ as follows:
\begin{align*}
    R(T) \leq C S^{3/2} \sqrt{\left(d\log(\frac{ST}{d}) + \log(\frac{T}{\delta})\right)\log\left(1+\frac{T}{\lambda \kappa d}\right) \frac{\kappa d}{T}}
\end{align*}
    \end{theorem}

\begin{proof}
    Let the suboptimality gap for a context $x \in \cX$ be denoted as $R(T,x)$. Thus,
    \begin{align*}
        R(T,x) &= \left(\phi(x,a^*(x)) - \phi(x,\pi_T(x))\right)^\intercal \theta^*\\
        &\leq \left(\phi(x,a^*(x)) - \phi(x,\pi_T(x))\right)^\intercal \theta^* + \left(\phi(x,\pi_T(x)) - \phi(x,a^*(x))\right)^\intercal \left(\frac{1}{T} \sum_{t=1}^T \widehat{\theta}_t \right)\tag{$\because \pi_T(x) = 
    \argmax_{a \in \cA} \phi(x,a)^\intercal \left( \frac{1}{T} \sum_{i=1}^T \widehat{\theta}_t \right)$}
\end{align*}
Therefore, we have,
\begin{align*}
        R(T) &\leq \left(\phi(x,a^*(x)) - \phi(x,\pi_T(x))\right)^\intercal (\theta^* - \frac{1}{T} \sum_{t=1}^T \widehat{\theta}_t)\\
        &= \frac{1}{T} \sum_{t=1}^T \left(\phi(x,a^*(x)) - \phi(x,\pi_T(x))\right)^\intercal (\theta^* - \widehat{\theta}_t)\\
        &\leq \frac{1}{T} \sum_{t=1}^T \lVert \phi(x,a^*(x)) - \phi(x,\pi_T(x)) \rVert_{H_t^{-1}(\widehat{\theta}_t)} \lVert \theta^* - \widehat{\theta}_t \rVert_{H_t(\widehat{\theta}_t)} ~.\tag{Holder's Inequality}
\end{align*}
Now we use lemma~\ref{lemma:bound-on-theta-hat-minus-theta-star-in-theta-hat-norm} to upper bound $\lVert \theta^* - \widehat{\theta}_t \rVert_{H_t(\widehat{\theta}_t)}$ with $CS^{\nicefrac{1}{2}}\gamma_t(\delta)$ which we further upper bound by $CS^{1/2}\gamma_T(\delta)$ after noting that $\gamma_t(\delta) \leq \gamma_{t+1}(\delta)$ for all $t \in [T]$. Thus, we now have,
\begin{align*}
        R(T,x) &\leq \frac{CS^{\nicefrac{1}{2}} \gamma_T(\delta)}{T} \sum_{t=1}^T \lVert \phi(x,a^*(x)) - \phi(x,\pi_T(x)) \rVert_{H_t^{-1}(\widehat{\theta}_t)}\\
        &\leq \frac{CS^{\nicefrac{1}{2}} \gamma_T(\delta)}{T} \sum_{t=1}^T \lVert \phi(x_t,a_{t}) - \phi(x_t,a'_{t}) \rVert_{H_t^{-1}(\widehat{\theta}_t)} \tag{$(x_t, a_{t}, a'_{t}) = \argmax_{x \in \cX, a,a' \in \cA} \lVert \phi(x,a) - \phi(x,a') \rVert_{H_t^{-1}(\widehat{\theta}_t)}$} \\
        &\leq \frac{C\sqrt{\kappa S} \gamma_T(\delta)}{T} \sum_{t=1}^T  \lVert \phi(x_t,a_{t}) - \phi(x_t,a'_{t}) \rVert_{V_t^{-1}} \tag{$V_t \preccurlyeq \kappa H_t(\widehat{\theta}_t)$}\\
        &\leq \frac{C\sqrt{\kappa S} \gamma_T(\delta)}{T} \sqrt{T \sum_{t=1}^T \lVert \phi(x_t,a_{t}) - \phi(x_t,a'_{t}) \rVert^2_{V_t^{-1}} } \tag{Cauchy-Schwarz}\\
        &\leq \frac{C\sqrt{\kappa S} \gamma_T(\delta)}{T} \sqrt{2dT \log\left(1+\frac{T}{\lambda \kappa d}\right)} \tag{Lemma~\ref{lemma:elliptic-potential-lemma}}\\
        &= C S^{3/2} \sqrt{\left(d\log(\frac{ST}{d}) + \log(\frac{T}{\delta})\right)\log\left(1+\frac{T}{\lambda \kappa d}\right) \frac{\kappa d}{T}} \tag{Def. of $\gamma_T(\delta)$}
\end{align*}
Thus, $R(T) = \max_{x \in \cX} R(T,x) \leq C S^{3/2} \sqrt{\left(d\log(\frac{ST}{d}) + \log(\frac{T}{\delta})\right)\log\left(1+\frac{T}{\lambda \kappa d}\right) \frac{\kappa d}{T}}$.
\end{proof}
\section{Generalization to Function Approximation: Proof of Theorem~\ref{thm:gen}}
\label{section:general-function-approximation}

In this section, we remove the assumption of the BTL preference model characterized by a linear parameter $\theta$. Instead, we assume that we have access to a function class $$\cF \!=\! \{f:\cX \!\times\! \cA \! \times\! \cA\rightarrow [0,1]:  f(x,a,a') \!+\! f(x,a',a) \!=\! 1\},$$ where $f(x,a,a')$ denotes the probability that the arm $a$ wins over arm $a'$ given context $x$ when the preference function is $f$, i.e., $f(x,a,a') = \PP[a \succcurlyeq a'|x,f]$ where $a \succcurlyeq a'$ denotes the event that $a$ wins over $a'$. Now, we assume that there is a true $f^* \in \cF$ from which the data is generated. Further, we assume a \textit{Condorcet} winner at each context:
\begin{assumption}
\label{assumption:condorcet-winner-general-function-approx}
    For all context $x \!\in\! \cX$, there is an action $a^*(x) \!\in\! \cA$ such that $f^*(x,a^*(x), a) \!\geq\! \nicefrac{1}{2}\, \forall a \!\in\! \cA$.
\end{assumption}
Note that in this case, there is no direct reward model and is, therefore a generalization of the BTL model. The absence of a reward model makes the problem more nuanced. Accordingly, the simple regret is now defined as:
\begin{align*}
    R(T) \!=\! \max\nolimits_{x \in \cX} \max\nolimits_{a \in \cA(x)} f^*(x, a, \pi_T(x)) \!-\! \nicefrac{1}{2}~.
\end{align*}
Note that $f^*(x, a^*(x), \pi_T(x)) \geq \nicefrac{1}{2}$ by assumption~\ref{assumption:condorcet-winner-general-function-approx}, thus $R(T)$ is always non-negative.
\subsection{Algorithm}
Our algorithm takes a function class $\cF$ and a confidence level $\delta \in (0,1]$ as its inputs.  
First, a regularized least square estimate of $f^*$ is computed by minimizing the cumulative squared prediction error:
\begin{align}\label{eq:gen-estimate}
 \widehat f_{t} \in \argmin\nolimits_{f \in \cF}   \sum\nolimits_{s=1}^{t-1} \left(y_s -f(x_s,a_s,a'_s)\right)^2~.
\end{align}
The confidence set $\cC_t(\cF,\delta)$ is then defined as
\begin{align}\label{eq:gen-conf}
      \cC_t(\cF, \delta) \coloneqq \left\{ f \in \cF: \sum\nolimits_{s=1}^{t-1} \!\!\big(f(x_s,a_s,a'_s)\!-\!\widehat f_t(x_s,a_s,a'_s)\big)^2 \!\le\! \beta_t(\cF,\delta)\right\},
\end{align}
where $\beta_t(\cF,\delta)$ is an appropriately chosen confidence parameter. Since $y_t \sim \texttt{Ber}(f^*(x_t,a_t,a'_t))$ given $(x_t,a_t,a'_t)$, We have $\Var [y_t] \le 1/4$. Thus, following \cite{ayoub2020model}, we set the confidence parameter
\begin{align*}
\beta_t(\cF,\delta)\!=\!2\log\frac{2\cN\!\left(\!\cF\right)}{\delta}\!+\!2\sqrt{\log \frac{4t(t\!+\!1)}{\delta}}\!+\!4~,
\end{align*}
where $\cN(\cF)$ denotes the $(1/t,\norm{\cdot}_{\infty})$-covering number\footnote{For any $\alpha > 0$, we call $\cF^\alpha$ an $(\alpha,\norm{\cdot}_{\infty})$ cover of the function class $\cF$ if for any $f \in \cF$ there exists an $f'$ in $\cF^\alpha$ such that $\norm{f' - f}_{\infty}:=\sup_{x \in \cX}|f'(x)-f(x)|\le \alpha$.} of $\cF$. This choice of confidence width ensures that $f^*$ lies in the confidence set $\cC_t(\cF,\delta)$ at all time instant $t \ge 1$ with probability at least $1-\delta$ (Lemma~\ref{lemma:func-approx-confidence-set}).

Next, for each triplet $(x,a,a')$, we define the exploration bonus $b_t(x,a,a')$ at round $t$ as 
\begin{align}\label{eq:gen-bonus}
   \!\! b_t(x,a,a') \!=\! \max_{f_1,f_2\in \cC_t(\cF,\delta)} \lvert f_1(x,a,a') \!-\! f_2(x,a,a') \rvert,
\end{align}
which measures
the uncertainty of a pair of actions $a,a'$ given a context $x$
with respect to the confidence set $\cC_t(\cF,\delta)$. The near-optimal
action set $\cA_t(x)$ at round $t$ is defined as the set of all actions in the previous  set $\cA_{t-1}(x)$ satisfying
\begin{align}\label{eq:gen-opt}
    \widehat{f}_t(x,a,a_0) + b_t(x,a,a_0)\geq \nicefrac{1}{2}\,
                \forall a_0 \in \cA_{t-1}(x)~.
\end{align}
Intuitively speaking, we retain only those actions from the previous near-optimal set that are not significantly outperformed by other actions according to the estimates of the current round. Since $f^* \in \cC_t(\cF,\delta)$, the optimal action $a^*(x)$ lies in $\cA_t(x)$ for each context $x$ for all $t$ with high probability (Lemma~\ref{lemma:a*-lies-in-A_t-gen-func-approx}). By pruning out suboptimal actions every round, we make better use of samples. When the set $\cA_t(x)$ becomes a singleton (i.e., $a^*(x)$ has been identified w.h.p), we remove this context from the pool of contexts considered in future rounds.

To encourage exploration, we choose actions 
$(a_t(x),a'_t(x))$ which has the highest uncertainty in $\cA_t(x)$, i.e., we choose
\begin{align}\label{eq:gen-action}
(a_t(x),a'_t(x))\!=\!\argmax\nolimits_{a,a' \in \cA_t(x)} b_t(x,a,a')~.
\end{align}
Next, we choose the context $x_t$ that provides the maximum information about the unknown preference function $f^*$, i.e.,
\begin{equation}\label{eq:gen-context}
    x_t \in \argmax\nolimits_{x \in \cX} b_t(x,a_t(x),a'_t(x))~.
\end{equation}
We play the actions $a_t=a_t(x_t)$ and $a'_t=a'_t(x_t)$ in round $t$ and observe the preference feedback $y_t$. We repeat this until we have exhausted the budget $T$. 
Our final policy $\pi_T$ samples an action uniformly at random from the set $\cA_T(x)$ for every context $x \in \cX$. Pseudocode is given in Algorithm~\ref{algo:act-con-sel-gen-func-approx}.

\subsection{Result}

We characterize the complexity of function class $\cF$ by its 
\emph{eluder dimension} \cite{russo2013eluder}.

\begin{definition}[Eluder dimension]\label{def:eluder}
The $\varepsilon$-eluder dimension $\dim_{\cE}(\cF,\varepsilon)$ of a function class $\cF$ defined on a domain $\cX$ is the length of the longest sequence $\lbrace x_i\rbrace_{i=1}^{n} \!\subseteq\! \cX$ of input points such that for some $\varepsilon' \!\ge\! \varepsilon$ and for each $ i \in\lbrace 2,\ldots,n\rbrace$,
\begin{align*}
    \sup_{f_1,f_2 \in \cF}\!\bigg\lbrace\! (f_1\!-\!f_2)(x_i) \!\;\Big | \; \! \!\sqrt{\sum\nolimits_{j=1}^{i-1}(f_1\!-\!f_2)^2(x_i)} \!\le\! \varepsilon'\!\bigg\rbrace \!>\! \varepsilon'~.
\end{align*}
\end{definition}
We denote by $d_{\cE}(\cF)=\dim_{\cE}\left(\cF,1/T\right)$, the $(1/T)$-Eluder dimension of the function class $\cF$. Now, we state the sub-optimality guarantee of the final policy using the eluder dimension and metric entropy of the function class $\cF$. This is a restatement of Theorem~\ref{thm:gen}.

\begin{theorem}[Suboptimality Gap]\label{thm:gen-restated}
    Let $\delta \in (0,1)$. Under assumption~\ref{assumption:condorcet-winner-general-function-approx}, the suboptimality gap $R(T)$ of our policy $\pi_T$ after running \texttt{APO-Gen} (algorithm~\ref{algo:act-con-sel-gen-func-approx}) for $T$ steps is upper bounded with probability at least $1-\delta$ as
    \begin{align*}
      R(T) \leq \tilde{O}\bigg(\sqrt{\frac{\log(\cN(\cF)T/\delta) d_\cE(\cF)}{T}}\bigg).  
    \end{align*}
\end{theorem}
Proof is deferred to the next section. It essentially follows ideas similar to Theorem~\ref{thm:log-restated} with the difference that we crucially leverage action elimination (Step~\ref{line:gfa-action-elimination}).

\textbf{BTL model.}
For the BTL preference model $f(x,a,a')=\mu(\phi(x,a)^\top \theta \!-\! \phi(x,a')^\top \theta)$. Define $r=\overline{h}/\underline{h}$, where
\begin{align*}
\overline{h}&= \sup\nolimits_{x,a,a',\theta}\, \dot{\mu}(\phi(x,a)^\top \theta \!-\! \phi(x,a')^\top \theta)~,\\ 
\underline{h}&= \inf\nolimits_{x,a,a',\theta} \,\dot{\mu}(\phi(x,a)^\top \theta \!-\! \phi(x,a')^\top \theta)~.
\end{align*}
Then the $\log\cN(\cF)$ and Eluder dimension of $\cF$ are at most $O(d\log(\overline{h} T))$ and $O(d r^2 \overline{h}\log(r S \overline{h} T))$, respectively. Note that $\underline{h}=1/\kappa$ and $\overline{h} \le 1/4$. This yields $\log \cN(\cF)=O(d\log T)$ and $d_\cE(\cF)=O(\kappa^2 d \log T)$. Substituting this in Theorem~\ref{thm:gen-restated}, we get the sub-optimality gap $O(\kappa d/\sqrt{T})$, which is $\sqrt{\kappa}$ factor loose than Theorem~\ref{thm:log-restated}. This is because we crucially use self-concordance of the sigmoid function in Theorem~\ref{thm:log-restated} to shave this extra $\sqrt{\kappa}$ factor. Nevertheless, Theorem~\ref{thm:gen-restated} is general enough to subsume other preference models (e.g., probit/Thurstone) beyond the BTL model.

\subsection{Analysis}

\begin{algorithm}[tb]
\small
\caption{\texttt{APO-Gen}: Active Preference Optimization with General Function Approximation}
\label{algo:act-con-sel-gen-func-approx}
\begin{algorithmic}[1]
    \REQUIRE Context set $\cX$, action set $\cA=[K]$, function class $\cF$, failure level $\delta \in (0,1)$.
    \STATE Set $\cX_0 = \cX$ and $\cA_0(x) = \cA \ \forall\ x \in \cX$.
    \FOR{$t=1, 2, \dots T$}
        \STATE 
        Compute function estimate $\widehat f_t$ usning~\eqref{eq:gen-estimate}.
        \STATE Construct confidence set $\cC_t(\cF,\delta)$ using~\eqref{eq:gen-conf}.
        \STATE Intialize the  $\cX_t = \cX_{t-1}$.
        \FOR{each context $x \in \cX_{t-1}$}
            \STATE For each pair of actions $a,a' \in \cA_{t-1}(x)$, compute the bonus $b_t(x,a,a')$ using~\eqref{eq:gen-bonus}. 
            \STATE Find the near-optimal action set $\cA_t(x)$ using~\eqref{eq:gen-opt}. \alglinelabel{line:gfa-action-elimination}
            \IF{$\lvert \cA_t(x) \rvert = 1$} 
            \STATE Set $\cA_T(x) = \cA_t(x)$ and $\cX_t \gets \cX_t \setminus \{x\}$.
            \ENDIF
        \ENDFOR
        \STATE Choose context and pair of actions
        $(x_t, a_t, a_t') = \argmax_{x \in \cX_t, a,a' \in \cA_t(x)}  b_t(x,a,a')$.
        \STATE Observe preference $y_t \sim \texttt{Ber}(f^*(x_t,a_t,a'_t))$
    \ENDFOR
    \STATE Output final policy $\pi_T(x) =a$ for some arbitrary $a \in \cA_T(x)$.
\end{algorithmic}
\end{algorithm}
First we present a result that characterizes the confidence set around $\widehat{f}_t$.
\begin{lemma}[Confidence Set for Function Approximation (Lemma A.1 of~\cite{chen2022human}]
\label{lemma:func-approx-confidence-set}
    Let $\delta \in (0,1)$. Define the confidence set $$\cC_t(\cF,\delta) = \{f \in \cF | \sum_{s=1}^{t-1} (f(x_s,a_s,a'_s) - \widehat{f}_t(x_s, a_s, a'_s))^2 \leq \beta_t(\cF, \delta) $$ Let $\cE_t(\delta)$ be the event that $f^* \in \cC_t(\cF,\delta)$. Then, $\PP[\cE_t(\delta)] \geq 1 - \delta$. Further, $\PP\left[\cap_{t=1}^T \cE_t(\delta/T)\right] \geq 1 - \delta$.
\end{lemma}
\begin{proof}
    The proof is a direct extension of lemma~\ref{lemma:function-approximation-confidence-set-basic-result-from-ayoub-et-al} by observing that in our case, the subgaussianity parameter $\sigma = \nicefrac{1}{4}$ since our rewards are Bernoulli and by setting $\alpha = \nicefrac{1}{t}$. Moreover, $C = 1$ in our case. Finally, since $\cE_t(\delta/t)$ holds with probability at least $1-\delta/t$, by union bound we can show that $\PP\left[\cap_{t=1}^T \cE_t(\delta/T)\right] = 1 - \PP\left[\cup_{t=1}^T \overline{\cE}_t(\delta/T)\right] \geq 1 - \sum_{t=1}^T \PP[\overline{\cE}_t(\delta/T)] \geq 1 - \delta$.
\end{proof}
Hereon, we will assume that $\cE_t(\delta/T)$ holds for all $t \in [T]$. All subsequent guarantees will be proved under this event. The next result shows that for each context $x$, the optimal action $a^*(x)$ lies in $\cA_t(x)$ for all $t$.
\begin{lemma}
\label{lemma:a*-lies-in-A_t-gen-func-approx}
    For a given context $x \in \cX$, let $\{\cA_s(x)\}_{s=0}^t$ be defined as follows: (a) $\cA_0(x) = \cA$ and (b) $\cA_s(x) = \{a \in \cA_{s-1}(x) \mid \widehat{f}_s(x,a,a') + b_s(x,a,a') \geq \frac{1}{2}\ \forall a' \in \cA_{s-1}(x) \} $. Then, we have, $a^*(x) \in \cA_s(x)$ for all $s \in [t]$.
\end{lemma}

\begin{proof}
    The proof is by induction. First note that by definition of $a^*(x)$, $f^*(x,a^*(x), a') \geq 1/2$ for every $a' \in \cA$, and $a^*(x) \in \cA_0(x) = \cA$. Suppose, for some $s > 0$, $a^*(x) \in \cA_{s-1}(x)$. Now, we know that under event $\cE_s(\delta/T)$, $f^* \in \cC_{s}(\cF,\delta/T)$ and thus from definition of $b_s(x,a,a')$, $f^*(x,a,a') - \widehat{f}_s(x,a,a') \leq b_s(x,a,a')$. Thus, for any $a' \in \cA_{s-1}(x)$,
    \begin{align*}
        \frac{1}{2} \leq f^*(x,a^*(x),a') \leq \widehat{f}_s(x,a^*(x),a') + b_s(x,a^*(x),a')
    \end{align*}
    Hence $a^*(x) \in \cA_s(x)$. Thus by induction, $a^*(x) \in \cA_s(x)$ for all $s \in [t]$.
\end{proof}
Now, we are ready to prove Theorem~\ref{thm:gen-restated}.
\begin{proof}[Proof of Theorem~\ref{thm:gen-restated}]
    The idea is to show that our arm elimination technique throws away arms with large suboptimality gaps in every round for every context. Thus, the set $\cA_{t}(x)$ maintains a candidate set of good arms at every time instant. In the end, playing any action from $\cA_T(x)$ ensures that we only play actions from a set of actions that are only $1/\sqrt{T}$ suboptimal. Formally, for any context $x \in \cX$, the suboptimality $R(T,x)$ is upper bounded as follows:
    \begin{align*}
        R(T,x) &= f^*(x, a^*(x), \pi_T(x)) - \frac{1}{2}\\
        &\leq \frac{1}{T}\sum_{t=1}^T \left[\widehat{f}_t(x, a^*(x), \pi_T(x)) + b_t(x, a^*(x), \pi_T(x)) - \frac{1}{2}\right] \tag{$a^*(x), \pi_T(x) \in \cA_{t}(x)\ \forall\ t \in [T]$}\\
        &= \frac{1}{T} \sum_{t=1}^T \left[ \frac{1}{2} -  \widehat{f}_t(x, \pi_T(x), a^*(x)) + b_t(x,a^*(x),\pi_T(x))\right] \tag{$f(x,a,a') + f(x,a',a) = 1 \ \forall\ f \in \cF$}\\
        &= \frac{1}{T} \sum_{t=1}^T \left[ \frac{1}{2} -  \widehat{f}_t(x, \pi_T(x), a^*(x)) + b_t(x,\pi_T(x),a^*(x))\right] \tag{$b_t(x,a,a')=b_t(x,a',a)$}\\
        &\leq \frac{1}{T} \sum_{t=1}^T \left[ b_t(x,\pi_T(x),a^*(x)) + b_t(x,\pi_T(x),a^*(x))\right] \tag{Since $\pi_T(x),a^*(x) \in \cA_{t}(x)$, line 7 Algorithm~\ref{algo:act-con-sel-gen-func-approx}}\\
        &= \frac{2}{T} \sum_{t=1}^T b_t(x,\pi_T(x),a^*(x)) \\
        &\leq \frac{2}{T} \sum_{t=1}^T b_t(x_t, a_t, a'_t)~. \tag{Line 9 of Algorithm~\ref{algo:act-con-sel-gen-func-approx}}
    \end{align*}
    Now we invoke lemma~\ref{lemma:function-approximation-bound-on-sum-of-diameters-similar-to-elliptic-potential-lemma} to bound the RHS.
    \begin{align*}
        R(T,x) \leq \frac{2}{T} \sum_{t=1}^T b_t(x_t, a_t, a'_t) \leq \frac{2}{T}\left[\frac{1}{T} + \min\{d_\cE(\cF), T\} + 2\beta_T(\cF,\delta/T) \sqrt{d_\cE(\cF) T}\right]
    \end{align*}
    Simplifying constants and using the fact that $\min\{a,b\}\leq \sqrt{ab}$ for $a,b>0$, we get $R(T,x) \leq C \beta_T(\cF,\delta/T)\sqrt{\frac{d_\cE(\cF)}{T}}$. Now, using order notation, we have for all $x \in \cX$ with probability at least $1-\delta$, $$R(T,x) \leq \tilde{O}\left(\sqrt{\frac{\log(\cN(\cF)T/\delta)d_\cE(\cF)}{T}}\right)~.$$
    Hence, $R(T) = \max_{x \in \cX} R(T,x) \leq \tilde{O}\left(\sqrt{\frac{\log(\cN(\cF)T/\delta)d_\cE(\cF)}{T}}\right)$.
\end{proof}

\section{Some Useful Results}
\label{section:useful-lemmas}
\begin{lemma}
\label{lemma:novel-self-concordance-analysis}
    Let $z, z' \in \RR$ and $\tilde{\alpha}(z,z') \coloneqq \int_0^1 (1-v) \dot{\sigma}(z + v(z' - z)) dv$. Then, for some $C>1$ ($1.01$ suffices), $$\tilde{\alpha}(z,z') \geq \frac{\dot{\sigma}(z')}{C(2+\lvert z - z'\rvert)^2}$$
\end{lemma}

\begin{proof} Firstly, note that by property of definite integrals $\int_a^b f(x)dx = \int_a^b f(a+b-x)dx$, we have $$\int_0^1 (1-v) \dot{\sigma}(z + v(z' - z)) dv = \int_0^1 v \dot{\sigma}(z' + v(z - z')) dv$$
Now, we use the fact that $\dot{\sigma}(x) \geq \dot{\sigma}(y) \exp(-\lvert x - y \rvert)$ (see appendix A of~\cite{faury2022jointly}). Let $a = \lvert z - z' \rvert$. Thus,
\begin{align*}
    \int_0^1 v \dot{\sigma}(z' + v(z - z')) dv &\geq \int_0^1 v \dot{\sigma}(z') \exp(-va) dv\\
    &= \dot{\sigma}(z') \int_0^1 v \exp(-va) dv\\
    &= \dot{\sigma}(z') \left(\frac{1- (1+a)e^{-a}}{a^2} \right)\\
    &\geq \dot{\sigma}(z') \left( \frac{1 - 1/a}{a^2} \right)  \tag{$(1+a)e^{-a} < \frac{1}{a}\ \forall a>0$}\\
    &= \dot{\sigma}(z') \left(\frac{a-1}{a^3}\right)
\end{align*}
Again, from appendix A of~\cite{faury2022jointly}, we have that $\tilde{\alpha}(z,z') \geq \dot{\sigma}(z)/(2+a)$ which can again be lower bounded with $\dot{\sigma}(z')e^{-a}/(2+a)$. Combining this lower bound with the above, we get,
\begin{align*}
    \tilde{\alpha}(z,z') \geq \max\left\{ \frac{e^{-a}}{2+a}, \frac{a-1}{a^3} \right\} \dot{\sigma}(z')
\end{align*}
Finally, we can lower bound the RHS with $\frac{\dot{\sigma}(z')}{C(2+a)^2}$ for some $C > 1.01$. To do this, let $f(x) = (2+x)e^{-x}$. Thus, $f'(x) = -(1+x)e^{-x}$ which implies that $f(x)$ is decreasing for $x > 0$. Thus,  $f(x) = \frac{1}{C}$ is satisfied for only one value of $x$ since $f(0)=2>1/C$. For $C=1.01$, this value is $x_0 = 1.1608$.  Then, for  $0 \leq x \leq  x_0$, $e^{-x}/(2+x) \geq  1/C(2+x)^2$.  Again, let $g(x) = (x-1)(x+2)^2/x^3$. Simplifying, we have, $g(x) = 1 + \frac{3}{x} - \frac{4}{x^3}$. It is easy to see that for $x \geq 2/\sqrt{3}$, $\frac{3}{x} \geq \frac{4}{x^3}$ which implies that $g(x) \geq 1$ for all $x \geq x_1 = 2/\sqrt{3}=1.1547$. So, for $x \geq 1.1547$, $g(x) \geq 1/C$ (since $C > 1$) which is equivalent to $\frac{(x-1)}{x^3} \geq \frac{1}{C(x+2)^2}$. The numeric solution to $g(x) = 1/C$ for $C=1.01$ is $x_2 = 1.1525$. It can be checked via the first derivative test that $g(x)$ is increasing in $x_2 \leq x \leq x_1$. Thus, indeed, $g(x) \geq 1/C$ for all $x \geq x_2$. Jence, we have established so far that for $C = 1.01$,
\begin{align*}
    \frac{x-1}{x^3} &\geq \frac{1}{C(x+2)^2} \quad\quad \forall\ x \geq x_2 = 1.1525\\
    \frac{e^{-x}}{2+x} &\geq \frac{1}{C(x+2)^2} \quad\quad \forall\ x \leq x_0 = 1.1608
\end{align*}
Since, $x_2 \leq x_0$, we have the required result that $\max\left\{ \frac{e^{-a}}{2+a}, \frac{a-1}{a^3} \right\} \dot{\sigma}(z') \geq \frac{\dot{\sigma}(z')}{C(2+a)^2}$ for all $a \geq 0$ which completes the proof.
\end{proof}

\begin{lemma}[Elliptic Potential Lemma]
\label{lemma:elliptic-potential-lemma}
Let $\{z_s\}_{s=1}^t$ be a sequence of vectors in $\RR^d$ such that $\lVert z_s \rVert \leq L$ for any $s \in [t]$. Let $V_t = \sum_{s=1}^{t-1} z_s z_s^\intercal + \lambda I$. Then,
$$\sum_{s=1}^t \lVert z_s \rVert^2_{V^{-1}_s} \leq 2d \log\left(1 + \frac{tL^2}{\lambda d}\right).$$
\end{lemma}
Now, we present the confidence set properties of function approximation. We use the same notations as~\cite{ayoub2020model}.

Let $(X_p, Y_p)_{p\geq1}$ be a sequence of random elements, $X_p \in \cX$ for some measurable set $\cX$ and $Y_p \in \RR$. Let
$\cF$ be a subset of the set of real-valued measurable functions with domain $\cX$. Let $\FF = (\FF_p)_{p\geq 0}$ be a filtration such that for all $p \geq 1, (X_1, Y_1, \dots , X_{p-1}, Y_{p-1}, X_p)$ is $\FF_{p-1}$ measurable and such that there exists some function $f^* \in \cF$ such that $\EE[Y_p \mid \FF_{p-1}] = f^*(X_p)$ holds for all $p \geq 1$. The (nonlinear) least-squares predictor given $(X_1, Y_1,\ldots, X_t, Y_t)$ is $\widehat{f}_t = \argmin_{f\in \cF} \sum^{t}_{p=1} (f(X_p) - Y_p)^2$. We say that $Z$ is conditionally
$\rho$-subgaussian given the $\sigma$-algebra $\FF$ if for all $\lambda \in \RR,\ \log \EE[\exp(\lambda Z)\mid \FF] \geq \frac{1}{2}\lambda^2 \rho^2$. For $\alpha > 0$, let $N_\alpha$ be the $\lVert \cdot \rVert_\infty$-covering number of $\cF$ at scale $\alpha$. That is, $N_\alpha$ is the smallest integer for which there exist $\cG \subset \cF$ with
$N_\alpha$ elements such that for any $f \in \cF$, $\min_{g\in \cG} \lVert f - g \rVert_\infty \leq \alpha$. For $\beta > 0$, define $\cF_t(\beta) = \{f \in \cF : \sum_{s=1}^t (f(X_s) - \widehat{f}_t(X_p))^2 \leq \beta\}$.

\begin{lemma}[Theorem 5 of~\cite{ayoub2020model}]
\label{lemma:function-approximation-confidence-set-basic-result-from-ayoub-et-al}
Let $\FF$ be the filtration defined above and assume that the functions in $\cF$ are bounded by the
positive constant $C > 0$. Assume that for each $s \geq 1$, $(Y_p - f^*(X_p))_p$ is conditionally $\sigma$-subgaussian given $\FF_{p-1}$. Then, for any $\alpha > 0$, with probability $1 - \delta$, for all $t \geq 1, f^* \in \cF_t(\beta_t(\delta, \alpha))$, where

$$\beta_t(\delta, \alpha) = 8\sigma^2 \log(2N_\alpha/\delta) + 4t\alpha\left(C + \sqrt{\sigma^2 \log(4t(t + 1)/\delta)}\right)~.$$

\end{lemma}

\begin{lemma}[Lemma 2 of~\cite{russo2013eluder}]
\label{lemma:function-approximation-bound-on-sum-of-diameters-similar-to-elliptic-potential-lemma}
Let $\cF \in B_\infty(\cX, C)$ be a set of functions bounded by $C > 0$, $(\cF_t)_{t\geq 1}$ and $(x_t)_{t\geq 1}$ be sequences such that $\cF_t \subset \cF$ and $x_t \in \cX$ hold for $t \geq 1$. Let $\cF\mid_{x_{1:t}} =
\{(f (x_1), \dots, f (x_t)) : f \in \cF \} (\subset \RR^t)$ and for $S \subset R^t$, let $diam(S) = \sup_{u,v\in S} \lVert u - v \rVert_2$ be the diameter of $S$. Then, for any $T \geq 1$ and $\alpha > 0$ it holds that
$$\sum^T_{t=1} diam(\cF_t\mid_{x_t}) \leq \alpha + C(d \wedge T ) + 2 \delta T \sqrt{dT}$$
where $\delta_T = \max_{1\leq t \leq T} diam(\cF_t\mid_{x_t})$ and $d = dim_\cE (\cF, \alpha) $ is the Eluder Dimension of $\cF$.
\end{lemma}
\section{Hyperparameter Details}
\label{appendix:experiment-details}
Any hyperparameters not mentioned use the default values in the TRL library.\footnote{\href{https://huggingface.co/docs/trl/index}{huggingface.co/docs/trl/index}} 

\begin{table}[!htbp]
\centering
\caption{Hyperparameters used in the experiment}
\label{tab:gpt2-exp}
\begin{tabular}{c c}
    \toprule
    Parameter & Value \\
    \toprule
    regularizer in \texttt{APO} & 1e-5 \\
    beta & 0.1 \\
    learning rate & 1.41e-5 \\
    batch size & 16 \\
    max length & 512 \\
    max prompt length & 128 \\
    \toprule
\end{tabular}
\end{table}
\textbf{Experiments on Controlled Sentiment Generation.}
All experiments here were run on a Tesla T4 (16GB) GPU. The hyperparameters for the experiments are outlined in Table \ref{tab:gpt2-exp}. 
\begin{table}[!htbp]
\centering
\caption{Hyperparameters used in PPO Training}
\label{tab:ppo-anthropic-details}
\begin{tabular}{c c || c c}
    \toprule
    Parameter & Value & Parameter & Value\\
    \toprule
    learning rate & 1e-4 & max new tokens & 256 \\
    lora-rank & 8 & top\_k & 80 \\
    lora\_alpha & 32 & top\_p & 1 \\
    lora\_dropout & 0.05 & temperature & 1.1 \\
    batch size & 16 & do\_sample & True \\
    mini batch size & 8 & &\\
    \toprule
\end{tabular}
\end{table}

\textbf{Experiment with Anthropic Dataset.}
All experiments here were run on an A100 (80 GB) GPU. For reward learning, we use regularizer $\lambda = 1\times 10^{-5}$. The learning rate for both \texttt{APO} and random is set to $1\times 10^{-2}$ with weight decay of $1\times 10^{-5}$. After every epoch of data collection, the training step on the logistic loss is run for $10$ epochs.
The details for the PPO configuration are the same for both \texttt{APO} and random. The details are given in Table~\ref{tab:ppo-anthropic-details}.



\end{document}